\documentclass{article}

\usepackage{arxiv}

\usepackage[utf8]{inputenc} 
\usepackage[T1]{fontenc}    
\usepackage{hyperref}       
\usepackage{url}            
\usepackage{booktabs}       
\usepackage{amsmath}
\usepackage{amsfonts}       
\usepackage{nicefrac}       

\usepackage{graphicx}
\usepackage{mathrsfs}
\usepackage{amsthm}
\usepackage{rotating}

\usepackage{algorithmicx}
\usepackage{algorithm}
\usepackage[noend]{algpseudocode}
\usepackage{array,xcolor,todonotes}
\usepackage{color, colortbl}
\usepackage[export]{adjustbox}
\usepackage{pifont}

\hypersetup{
    colorlinks=true,       
    linkcolor=red,          
    citecolor=blue,        
    filecolor=cyan,         
    urlcolor=blue        
}

\newcolumntype{C}[1]{>{\centering}m{#1}}
	
\definecolor{Gray}{gray}{0.9}

\newcommand{\cmark}{{\color{black!50!green}\bfseries\ding{51}}}%
\newcommand{\xmark}{{\color{red}\bfseries\ding{55}}}%

\newcommand{\mysmall}[1]{
\begin{center}
{\small #1}
\end{center}
}

\DeclareMathOperator*{\argmax}{arg\,max}
\DeclareMathOperator*{\argmin}{arg\,min}
\newcommand{\E}{\mathbf{E}}
\newcommand{\Id}{\mathrm{I}}

\def\R{{\mathbf R}}
\def\calX{{\mathcal X}}
\def\calY{{\mathcal Y}}

\def\calZ{{\mathcal Z}}

\def\OTc{\mathrm{OT}_c}
\newlength{\flen}
\def\LL{N_F} 
\newcommand{\vv}{u_0} 
\newcommand{\FF}{\mathcal{J}}
\newcommand{\GM}{Gr}
\newcommand{\mixed}{mix}

\newtheorem{hyp}{Hypothesis}
\newtheorem{rem}{Remark}
\newtheorem{theorem}{Remark}
\graphicspath{{figs/}}

\title{A Generative Model for Texture Synthesis based on Optimal Transport between Feature Distributions\thanks{This study has been carried out with financial support from the French Research Agency through the GOTMI project (ANR-16-CE33-0010-01). The authors also acknowledge the French GdR ISIS through the support of the REMOGA project.}}

\date{}

\author{Antoine Houdard\\
	Université de Bordeaux\\
	Bordeaux INP, CNRS, IMB, UMR 5251\\
	33400 Talence, France\\
	\texttt{antoine.houdard@ubisoft.com}\thanks{Current affiliation: Ubisoft La Forge, 33000 Bordeaux}\\	
	\And 	
	Arthur Leclaire\\
	Université de Bordeaux\\
	Bordeaux INP, CNRS, IMB, UMR 5251\\
	33400 Talence, France\\
	\texttt{arthur.leclaire@u-bordeaux.fr}\\	
	\And	
	Nicolas Papadakis\\
	Université de Bordeaux\\
	Bordeaux INP, CNRS, IMB, UMR 5251\\
	33400 Talence, France\\
	\texttt{nicolas.papadakis@u-bordeaux.fr}\\	
	\And	
	Julien Rabin\\
	Université de Normandie\\
	UniCaen,  ENSICAEN, CNRS, GREYC, UMR 6072\\
	14000 Caen, France\\
	\texttt{julien.rabin@unicaen.fr}\\}

\renewcommand{\headeright}{}
\renewcommand{\undertitle}{}
\renewcommand{\shorttitle}{GOTEX: a Generative model based on Optimal transport for synthesizing TEXtures}

\begin{document}

\maketitle

\begin{abstract}
We propose GOTEX, a general framework for texture synthesis by optimization that constrains the statistical distribution of local features. While our model encompasses several existing texture models,
we focus on the case where the comparison between feature distributions  
relies on optimal transport distances.
We show that the semi-dual formulation of optimal transport allows to control the distribution of various possible features, even if these features live in a high-dimensional space.
We then study the resulting minimax optimization problem, which corresponds to a Wasserstein generative model, for which the inner concave maximization problem can be solved with standard stochastic gradient methods.
The alternate optimization algorithm is shown to be versatile in terms of applications, features and architecture; in particular it allows to produce high-quality synthesized textures with different sets of features.
We analyze the results obtained by constraining the distribution of patches or the distribution of responses to a pre-learned VGG neural network. We show that the patch representation can retrieve the desired textural aspect in a more precise manner. 
We also provide a detailed comparison with state-of-the-art texture synthesis methods.
The GOTEX model based on patch features is also  adapted to texture  inpainting and texture interpolation.
Finally, we show how to use our framework to learn a feed-forward neural network that can synthesize on-the-fly new textures of arbitrary size in a very fast manner. Experimental results and comparisons with the mainstream methods from the literature illustrate the relevance of  the generative models learned with  GOTEX. 
\end{abstract}


%
\section{Introduction}
%

A lot of attention has been recently drawn on the problem of designing deep generative models from an image database~\cite{goodfellow2014gan,arjovsky2017wgan,karras2019style}.
In contrast, synthesizing a texture from a single sample is a long-standing image processing problem for which many solutions have been proposed, as we will recall below.
The main purpose of this work is to discuss whether the methodology developed for deep generative models can adapt to the case of learning from a single texture sample, depending on the choice of textural features that one wishes to preserve.
We will restrict to the relatively simple case of stationary textures (i.e. with no large geometric deformations nor lighting changes) which already benefits from powerful tools for analysis and synthesis.

\subsection{Features for texture synthesis}
In the stationary setting, the common point of view adopted in parametric texture models is to represent the textural aspect through the statistics of local features extracted from the neighborhoods of all pixels. 
Parametric texture models thus encompass the Gaussian model~\cite{galerne_rpn_2011} (based on mean and covariances of pixel values), the Heeger-Bergen model~\cite{heegerbergen_1995} (based on first-order distributions of responses to a filter bank) and the Portilla-Simoncelli approach~\cite{portilla_simoncelli_2000} (based on second-order statistics computed on complex wavelet filter responses).

More recently, features extracted with a deep convolutional neural network have permitted to accurately solve difficult imaging problems, with tremendous success in image classification~\cite{krizhevsky2012imagenet,VGG} or texture synthesis~\cite{gatys_texture_2015} for example.
Such {\em deep features} are nevertheless complex to understand and to interpret. This makes difficult the prediction and the tuning of the results provided by methods based on deep features.
An illustration of this major caveat is that, among all existing representation learning techniques, the only pre-trained neural features that are used in practice for texture synthesis (e.g. in~\cite{johnson2016perceptual,ulyanov2016texture}) are solely based on the VGG network trained on ImageNet~\cite{VGG}, as proposed in the seminal work of~\cite{gatys_texture_2015}. 
As shown in our experiments of section~\ref{sec:generator}, Adversarial-based techniques are not competitive when training on a single image.
Additionally, those features require GPUs with large memory to be computed efficiently.
A question that naturally arises is then: do we actually \emph{need} deep features to encode a texture?

Deep features are computed from the image on \emph{patches}, which are small regions of size $s\times s$ around each pixel, also called the local receptive field of the feature. 
Patches of pixels are the simplest local feature that can be considered in this setting. 
Such a patch representation has originally be proposed to design texture synthesis methods based on simple iterative copy/paste operations or nearest-neighbor assignments~\cite{efros-nonparam-sampling,Kwatra}.
The patch representation has also been widely exploited for other purposes. 
It is indeed at the core of efficient image restoration methods~\cite{buades2005nlmeans,lebrun2013nlbayes,houdard2018high}.
Recently, it has also been shown to be powerful in comparison to representation learning techniques~\cite{thiry2021unreasonable}.

\paragraph{Current limitations}
Patch-based approaches generally suffer from three main practical limitations. First, the patches are often processed independently and then combined to form a recomposed image~\cite{galerne2018texture,leclaire_multilayer_ssvm19}. The overlap between patches leads to low frequency artifacts such as blurring. Second, the optimization has to be performed sequentially in a coarse-to-fine manner (both in image resolution and patch size) starting from a good initial guess. Last, global patch statistics must be controlled along the optimization to prevent strong visual artifacts~\cite{Gutierrez_ssvm2017,kaspar2015self}.

In the deep neural network community, deep feature representations have overtaken the patch representation in most of recent texture synthesis methods. Patches may indeed be considered to be less informative than deep features. Popular texture synthesis methods such as~\cite{gatys_texture_2015}, which enforces the Gram matrices of deep features from the synthesized texture, do not provide meaningful results if deep features are replaced by patches.
Nevertheless, the use of deep features leads to visual artifacts such as color inconsistencies or checkerboard patterns on the generated texture. Post-processing steps such as histogram equalization as in~\cite{gatys_texture_2015} or the application of median filter are necessary to provide relevant synthesis~\cite{durand2021shallow}.

\subsection{Optimal transport for texture synthesis}
Should we be working with patches or deep features, one common difficulty is to design tools that allow to compare the distribution of patches or feature responses (which both live in a high-dimensional space and have a strongly non-Gaussian behavior). 

In this work, we propose to compare these distributions with an optimal transport (OT) cost. Contrary to  divergences (e.g. Kullback-Leibler), the OT distance is a relevant tool for comparing distributions that have disjoint supports.  It is also  adapted for matching both discrete and continuous distributions. As we now detail, the use of OT cost for texture synthesis has already proven to be fruitful in the literature. 

For example, the authors of~\cite{tartavel2016wasserstein} suggest to rely on discrete Wasserstein distances in order to measure the proximity of distributions of extracted features (thus reinterpreting the Heeger-Bergen algorithm~\cite{heegerbergen_1995} as an alternate gradient descent on a composite Wasserstein cost).
In~\cite{rabin2011wasserstein} a sliced Wasserstein distance is used on distributions of local features (responses to a steerable pyramid) in order to compute texture barycenters.
Notice also that such a sliced Wasserstein distance was used in~\cite{heitz2021sliced} to compare distributions of deep features, in order to address texture synthesis.
Wasserstein distances can also be used to compare patch distributions, either with a discrete formulation~\cite{Gutierrez_ssvm2017} or a semi-discrete one~\cite{galerne2018texture,leclaire_multilayer_ssvm19}.
In~\cite{vacher2020texture}, the authors proposed to extract the means and covariances of the feature responses and then to rely on the Wasserstein distance between the corresponding Gaussian distributions. 
This method exploits the closed-form formula of the Wasserstein distance between Gaussian distributions, which can be efficiently computed in any dimension (as already used in~\cite{xia2014synthesizing} for dynamic textures). 
This can be seen as an extension of the model from~\cite{gatys_texture_2015} that compares feature distributions by exploiting the Frobenius norm between Gram matrices of features. However, the approach in~\cite{vacher2020texture} cannot handle non-Gaussian behavior of the feature responses. In the following, we will instead consider the general optimal transport case with no assumption on the feature distribution.

\paragraph{Current limitations of Wasserstein Generative models}
In parallel, the use of Wasserstein distances has helped to improve models based on adversarial training.
As introduced in~\cite{goodfellow2014gan} for image synthesis from a database, a generative adversarial network (GAN) is inherently trained to fool another neural network that is simultaneously optimized to discriminate between real images and synthetic images. 
Such an adversarial training can be formulated as the minimization of a discrepancy between distributions, namely the Jensen-Shannon divergence in the original work~\cite{goodfellow2014gan} or the $1$-Wasserstein distance in the paper~\cite{arjovsky2017wgan} introducing Wasserstein generative adversarial networks (WGAN).
Both these works rely on a dual formulation of the chosen discrepancy and suggest to parameterize the corresponding dual variable by a neural network.
Thanks to the properties of the Wasserstein distance,
WGAN has offered an elegant solution to mode collapse issues related to GANs.
Alternative techniques to train generative neural networks also took profit from using the OT framework. 
For instance, the Sliced-Wasserstein distance has been considered in the latent space of auto-encoders in~\cite{kolouri2018sliced}.

Building on these ideas,  adversarial models have been proposed for texture synthesis from a single example~\cite{bergmann2017learning} or for feed-forward synthesis of general images~\cite{shaham2019singan}.
Although achieving convincing performance on synthesis problems, the main limitation of GAN or WGAN is that they require to optimize a discriminative network, which makes the process unstable and requires a large number of additional parameters~\cite{goodfellow2014gan,mescheder2018training}. 
In the case of WGAN, the discriminative network is theoretically related to the dual formulation of the Wasserstein-1 distance, and thus should represent a 1-Lipschitz mapping. 
Different strategies have thus been proposed to enforce such a constraint (\emph{e.g.} weight clipping or gradient penalty~\cite{gulrajani2017improved}), thus only approximating the true Wasserstein-1 distance. 
Another strategy adopted by the authors of~\cite{seguy2018large} is to rely on regularized optimal transport, which leads to an unconstrained dual problem.
However, this new dual problem involves two dual variables that must be parameterized by two different neural networks, which leads to a non-convex problem with twice more variables.
In contrast, in~\cite{chen2019gradual}, the optimization of the Wasserstein distance in WGAN is driven by the semi-discrete formulation of OT between the discrete distribution of training images and the density of generated images.
This has the benefit of keeping a convex formulation for the OT dual problem which stabilizes training~\cite{Houdard_ssvm21}, while not being specific to the $L^1$ cost.
In the following we will adopt the same approach than~\cite{chen2019gradual,Houdard_ssvm21} to approximate the solution of OT,
but we will include it in a more general framework able to learn a generative network. 
In addition, the OT distances will be used not to compare distributions of generated images but rather to constrain the feature distribution of synthesized images.

\begin{figure}[t]
\centering
\includegraphics[width=\linewidth]{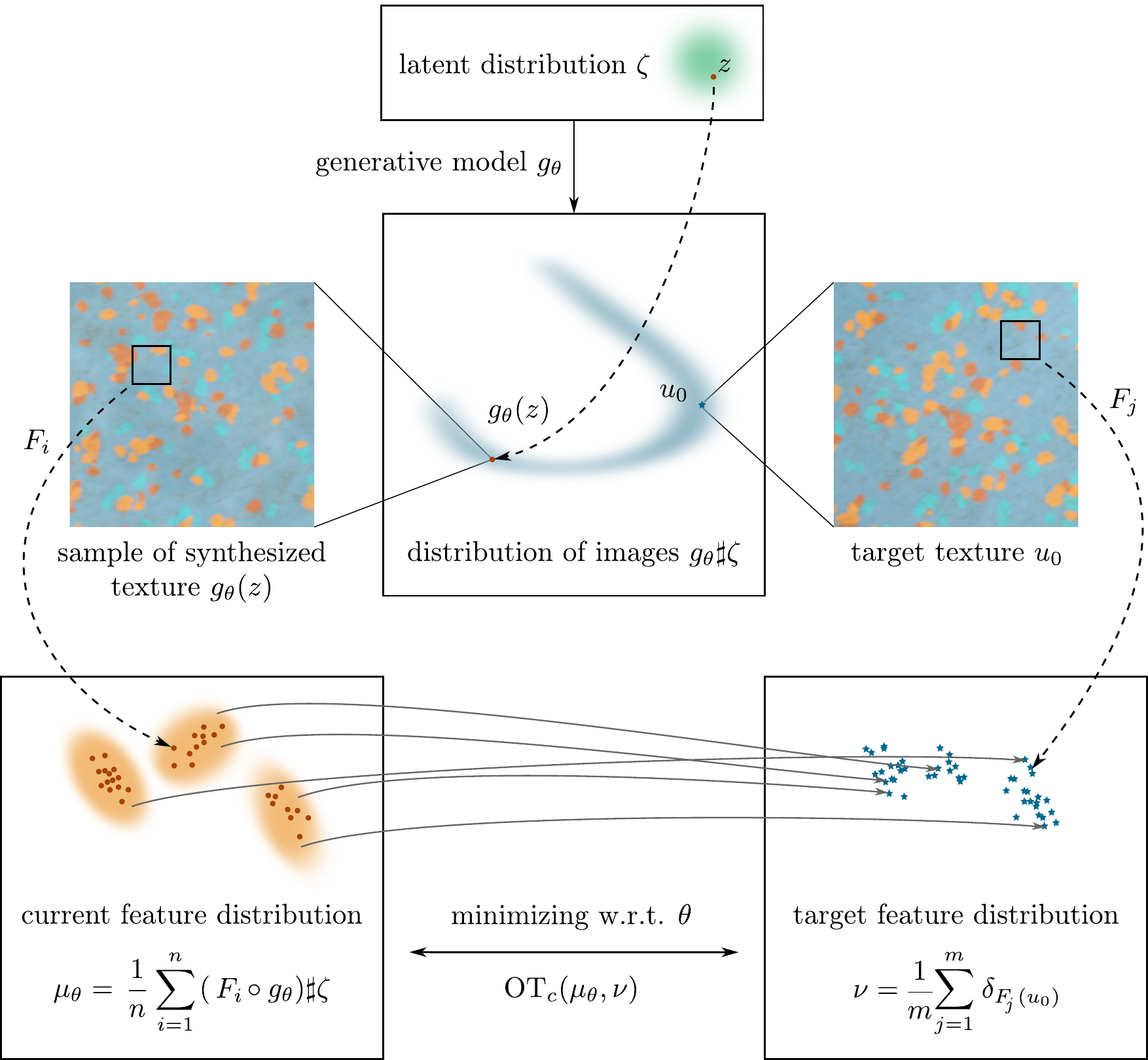}
\caption{Summary of the proposed GOTEX framework. A texture formation model is encoded with a generative model $g_\theta$ and the distribution of texture images is represented through its feature distribution $\mu_\theta$. The objective is then to minimize the optimal transport cost $\OTc(\mu_\theta,\nu)$ between the current feature distributions $\mu_\theta$ and the (discrete) feature distribution $\nu$ of the example target texture $u_0$. This framework also encompasses the case where the optimization is done on the image pixels by taking the latent distribution as a Dirac (see section \ref{sec:image-optim} for details).
}\label{fig:accroche}
\end{figure}

\begin{table}[htb]
\centering
\caption{Technical comparison of previous work based on the following criteria:
Fast synthesis based on a \emph{Feed-Forward} architecture trained offline;
Optimal Transport (\emph{OT}) based objective function;
\emph{Patch}-based representation;
\emph{Deep} features for optimization, 
where * indicates that such features are simultaneously learned during training using an adversarial loss, rather than defined from a pre-trained neural network.}
\label{tab:classification_previouswork}
\begin{tabular}{r||cccc}
Method & Feed-Forward & OT & Patch & Deep Features \\
Gram-VGG~\cite{gatys_texture_2015} & \xmark & \xmark & \xmark & \cmark\\
SINGAN~\cite{shaham2019singan} & \cmark & \xmark & \cmark & \cmark*\\
PSGAN~\cite{bergmann2017learning}& \cmark & \xmark & \cmark & \cmark*\\
Texture Networks~\cite{ulyanov2016texture} & \cmark & \xmark & \xmark & \cmark\\
TexOptim~\cite{Kwatra} & \xmark & \xmark & \cmark & \xmark \\
OPA~\cite{Gutierrez_ssvm2017} & \xmark & \cmark & \cmark & \xmark\\
TexTo~\cite{leclaire_multilayer_ssvm19} & \cmark & \cmark & \cmark & \xmark \\
GOTEX & \cmark & \cmark & \cmark & \cmark \\
\end{tabular}
\end{table}

\subsection{Contributions and outline} 

As summed up in Table~\ref{tab:classification_previouswork}, state-of-the-art texture synthesis methods are either based on patches or deep features. 
In this work, we propose a unified framework in order to address the limitations associated with both kinds of features with a formulation inspired by generative networks.

For this purpose, we introduce a Generative model based on Optimal transport for synthesizing TEXtures (GOTEX) while prescribing their feature distributions. 
The idea is to define a texture formation model as the push-forward of a latent distribution $\zeta$ by a measurable function $g_\theta$ and to consider its underlying feature distribution. 
Then the parameter $\theta$ of the model is optimized to enforce the feature distribution at different scales to be close to the one of the exemplar image, in the sense of optimal transport. The proposed pipeline is illustrated in Fig.~\ref{fig:accroche}. 
The organization of the paper and the description of its main contributions are listed below. 

In section~\ref{sec:texturemodel} we introduce the GOTEX framework that enforces the feature distribution of generated textures and treats in the same way texture synthesis by pixel-wise optimization (section \ref{ssec:image-optim}) and by learning a generative model (section \ref{ssec:gen_models}). 
Both problems will respectively rely on the discrete and semi-discrete formulations of the optimal transport cost.
In section~\ref{sec:theo} we state a differentiation result that gives a formula for the gradient of the optimal transport cost between feature distributions with respect to the parameter $\theta$ (see theorem~\ref{thm:2}).

In section~\ref{sec:feat} we present the GOTEX algorithm and detail the versatility of the framework, which can combine different distributions of features in the texture synthesis model. 
Our approach namely encompasses multi-scale procedures on patches or VGG-19 features. 
Contrary to previous methods relying on approximations of OT that are detailed in section \ref{sec:previous_work_OT}, our framework allows for a more accurate optimization. The proposed model also offers a theoretical-sound framework to compute barycenters of texture models, thus providing a relevant way to synthesize interpolated textures.

In section \ref{sec:image-optim}, we focus on the image optimization setting where an optimization problem is solved for each new synthesis. This involves  discrete optimal transport problems that can be efficiently solved with dedicated nearest-neighbor search libraries. We then propose an extensive analysis of the model including the  comparison of different features and losses and the comparison of different numerical methods approximating the optimal transport. In addition to texture synthesis, we also generalize the model in order to address texture inpainting and interpolation.

In section \ref{sec:generator}, we finally demonstrate that the GOTEX framework is well suited  to the training of a deep generative feed-forward convolutional neural network, as proposed in~\cite{ulyanov2016texture,bergmann2017learning,shaham2019singan} for texture generation.

%
\section{GOTEX: a  Generative model based on Optimal transport for synthesizing textures}
\label{sec:texturemodel}
%
In this section, we present a generic framework formulating texture synthesis as the minimization of a loss function that reflects the proximity of a set of features of the synthesized image(s) to the ones of the example.
As we will see, different choices of the loss functions can model the statistical behavior of the features in a parametric or non-parametric way.
As we will see,  when considering the feature distributions, the loss function can be expressed using optimal transport distances. %

\subsection{Texture synthesis by minimizing a distance between feature sets}\label{ssec:image-optim}
We first consider the synthesis of a single image $u \in \R^n$ with $n$ pixels.
For each pixel~$i$ of the image, we consider a measurable map $F_i : \R^n \to \R^d$ that extracts a local feature of dimension $d$ computed from the neighborhood of pixel $i$.
For example, $F_i(u)$ may be a square patch of dimension $d = 3 \times s \times s$, or a collection of $d$ neural responses computed at pixel $i$.
The features of the image $u$ will be gathered in a vector $F(u) = (F_i(u))_{1 \leq i \leq n} \in \R^{dn}$.

We also consider a cost function $\Lambda : \R^{dn} \times \R^{dn} \to \R_+$ that is chosen to assess the proximity between feature maps.
Then we can define a loss function between two textures $u\in\R^n$ and $\vv\in\R^n$ as
\begin{equation}
\label{eq:loss}
\mathcal{L}(u,\vv) = \Lambda\left( F(u),F(\vv)\right).
\end{equation}

A new sample $u$ of a given example texture $\vv$ may be obtained by minimizing~\eqref{eq:loss} with respect to $u$. 
If we assume that $\Lambda$ is differentiable and that for all $i$, $F_i$ is differentiable with respect to the image $u\in\R^n$, the loss function \eqref{eq:loss} may be minimized by  performing a gradient descent with respect to the pixels of the image~$u$. 
Due to the potential non linearity of the feature extraction operators $F_i$ and/or of the loss function~$\Lambda$ such a problem is typically non-convex.
However, starting from a random initialization, gradient descent schemes can converge to local minima that will correspond to plausible syntheses of the exemplar texture.

\begin{rem}The seminal work of~\cite{gatys_texture_2015} is included in this framework. More precisely, it corresponds 
to take as features $F_i$(u) the normalized outputs of a pre-trained VGG network~\cite{VGG} at different layers $l$, 
to define the Gram matrix of the features
\begin{equation}
\GM(F(u)) = \frac{1}{n}\sum_{i=1}^n F_i(u)F_i^{\top}(u) \in \R^{d \times d},
\end{equation}
and to compare Gram matrices with the squared Frobenius norm $\|.\|_F^2$.
A synthesized image is then obtained in~\cite{gatys_texture_2015} by minimizing with respect to $u$ the quantity
\begin{equation}\label{eq:gram_loss}
\mathcal{L}_{Gram}(u,\vv) = \sum_l \|\GM(\textrm{VGG}_l (u)) - \GM(\textrm{VGG}_l(\vv))\|_F^2.
\end{equation}
\end{rem}

\subsection{Generative models}\label{ssec:gen_models}
The previous model \eqref{eq:loss} requires to perform an optimization each time a new texture $u$ is synthesized. Hence, the authors of~\cite{ulyanov2016texture} have later proposed to first train a generative model and then realize new syntheses on-the-fly. The optimization is realized on the parameters of a feed-forward network rather than on the image pixels.

To that end, we assume that different samples of a given texture are actually samples of a probability distribution. This distribution is defined as the push-forward of a given random distribution $\zeta$  defined on $\calZ$ (e.g. a uniform distribution $\zeta$ on a latent space $\calZ=[0,1]^M$ of dimension $M$), with a generator $g$ to estimate.
Let us consider a measurable generative function $g :\Theta\times\calZ\to\R^n$ where $\Theta$ is a set of parameters. For a given parameter $\theta\in\Theta$ we write $g_\theta = g(\theta, \cdot)$ and we consider the output texture distribution as the push-forward $g_\theta\sharp\zeta$, which is given by $g_\theta\sharp\zeta(B) = \zeta(g_\theta^{-1}(B))$ for any Borel set $B$. A relevant generative model may thus be learned by minimizing with respect to the parameters $\theta$ the following objective function:
\begin{equation}\label{eq:loss_gen}
\mathcal{L}_{gen}(\theta,\vv) = \E_{Z\sim\zeta} \left[ \Lambda \left( F(g_\theta(Z)), F(\vv)\right) \right].
\end{equation}
When considering such generative models, we face a {\em semi-discrete problem} as the data at hand $\vv$ is discrete while the generated distribution $g_\theta\sharp\zeta$ is expected to be absolutely continuous.
In the following, we present a general framework that includes both models \eqref{eq:loss} and \eqref{eq:loss_gen} based on the probabilistic distributions of features.

\subsection{Probabilistic representation of the features}
We now define the generic texture formation model that permits to  encompass the optimization on the image pixels \eqref{eq:loss} and the  optimization on the weights of a generative model \eqref{eq:loss_gen}.
To that end, we propose to consider the probability distribution of the features $F_i$ taken on the texture distribution $g_\theta\sharp\zeta$.
Assuming that all $F_i$ are  measurable, we have for each $i$ a local feature distribution given by
\begin{equation}
\mu_\theta^i = F_i\sharp\left(g_\theta\sharp\zeta\right) = \left(F_i\circ g_\theta\right)\sharp\zeta.
\end{equation}
Then the whole feature distribution of the generative model writes
\begin{equation}\label{eq:def_mu}
\mu_\theta = \frac{1}{n}\sum_{i=1}^n  \mu_\theta^i = \frac{1}{n}\sum_{i=1}^n \left(F_i\circ g_\theta\right)\sharp\zeta.
\end{equation}

\begin{rem}\label{rem:si}
This formulation includes the case of the single image synthesis corresponding to the minimization with respect to image pixels \eqref{eq:loss}. If $\theta$ denotes the image to optimize, then taking  $\zeta = \delta_0$ and $g(\theta, z) = \theta - z$ gives $g_\theta\sharp\zeta = \delta_\theta$. Since we have $F_i\sharp\delta_\theta = \delta_{F_i(\theta)}$, the underlying feature distribution of the image $\theta$ writes as the discrete probability distribution
\begin{equation}
\mu_\theta = \frac{1}{n}\sum_{i=1}^n  \delta_{F_i(\theta)}.
\end{equation}
\end{rem}
To perform texture synthesis, we aim at minimizing with respect to $\theta$ a distance function between $\mu_\theta$ and the distribution $\nu$ of features extracted from a target image $\vv$. In this setting, a natural tool for comparing probability distributions appears to be optimal transport. We now describe our  formulation based on an optimal transport distance. 

\subsection{Optimal Transport cost for comparing feature distributions}
Given an example texture $\vv$, we follow the previous section and denote as $\nu$ its feature distribution:
\begin{equation}
\nu = \frac{1}{m}\sum_{j=1}^m \delta_{F_j(\vv)}. 
\end{equation}
Our objective is  to constrain the feature distribution $\mu_\theta$ of the synthesized textures defined in \eqref{eq:def_mu} in order to match the target distribution $\nu$. To do so, in the GOTEX framework, the loss function forces the patch distribution $\mu_\theta$ of the synthesized textures to be close to the empirical example patch distribution $\nu$ for the optimal transport cost

\begin{equation}
\mathcal{L}_\text{GOTEX}(\theta,\vv) = \mathrm{OT}_c(\mu_\theta,\nu) = \inf_{\pi\in\Pi(\mu_\theta,\nu)} \int c(x,y)d\pi(x,y), \label{eq:OTcost}
\end{equation}
where $c : \R^d \times \R^d \to \R$ is a Lipschitz cost function (between features) and $\Pi(\mu_\theta,\nu)$ is the set of probability distributions on $\R^d \times \R^d$ having marginals $\mu_\theta$ and $\nu$. 
When using $c(x,y) = \|x-y\|^2$, as done for experiments in this paper, $\OTc$ corresponds to the square of the Wasserstein-2 distance.

Minimizing the optimal transport cost in Equation~\eqref{eq:OTcost} with respect to one of its argument is a difficult task in general. The situation is even harder in our case as we wish to differentiate~\eqref{eq:OTcost} with respect to $\theta$ and we have to deal with a nonlinear mapping $\theta\mapsto g_\theta$. 
The dual formulation of OT will allow us to separate the problems of approximating the OT distance and minimizing w.r.t $\theta$ as in~\cite{arjovsky2017wgan}. 
In this work, we will also exploit the discrete nature of the target distribution $\nu$ to rely on flexible algorithms for semi-discrete optimal transport. Before going into such technical details in the next section, we present below in \eqref{eq:minimax} the final problem we will optimize. 

Texture synthesis is obtained with the minimization of the optimal transport cost~\eqref{eq:OTcost} with respect to $\mu_\theta$, its first argument. Hence we consider the  semi-dual formulation of the optimal transport cost.

\begin{theorem}[Semi-dual formulation~\cite{santambrogio2015ot}] \label{th:ot_semidual}
If $\calX$ and $\calY$ are compact and the cost $c$ is continuous, then
  \begin{equation}
    \label{eq:semidualpsi}
    \mathrm{OT}_c(\mu,\nu) = \max_{\psi \in \mathcal{C}(\mathcal{Y})} \int \psi^{c}(x) d \mu(x) + \int \psi(y) d\nu(y),
  \end{equation}
where $\psi:\calY \to \R$ and its $c$-transform is defined by 
\begin{equation}
\psi^{c}(x) = \min_{y \in \calY} \left[ c(x,y) - \psi(y)\right].
\end{equation}
\end{theorem}

Semi-dual here refers to the fact that the dual problem is formulated with only one dual variable while the other dual variable is optimized through the $c$-transform.
Combining \eqref{eq:OTcost} and \eqref{eq:semidualpsi}, we get that 
estimating the variable $\theta$ (a new texture or the parameters of a generator) amounts to solving the following  problem
\begin{equation}\label{eq:minimax}
\min_\theta \mathrm{OT}_c(\mu_\theta,\nu)=\min_\theta \max_{\psi \in \mathcal{C}(\mathcal{Y})} \FF(\theta,\psi):= \int \psi^{c}(x) d\mu_\theta(x)+ \int \psi(y) d\nu(y).
\end{equation}

For a fixed $\theta$, the function $\psi \mapsto \FF(\theta,\psi)$ is concave and an optimal $\psi^*$ can be approximated with an averaged stochastic gradient ascent as proposed in~\cite{genevay2016ot}.

When the variable $\theta$ is an image, we propose in Section~\ref{sec:image-optim} to perform a gradient descent, whose outcome is illustrated in Fig.~\ref{fig:accroche}.
A stochastic gradient-based algorithm is finally proposed in Section~\ref{sec:generator} to learn a generative model using a convolutional neural network.
Both approaches exploit the property, demonstrated in Section~\ref{sec:theo}, that, upon existence, the gradient of the optimal transport $\nabla_\theta \mathrm{OT}_c(\mu_\theta,\nu)$ coincides with the gradient $\nabla_\theta \FF(\theta,\psi^*)$ of the function $\FF$ at an optimal value $\psi^*$. 

%

\section{Gradients for the Semi-Dual Optimal Transport Cost}\label{sec:theo}

In this section we study the gradients with respect to $\theta$ and $\psi$ of the optimal transport cost $\FF(\theta,\psi)$ introduced in~\eqref{eq:minimax}.
In the whole section, we will assume that all feature extraction operators  $F_i : \R^m \to \R^d$ are differentiable.

\subsection{The semi-discrete formulation} 
When dealing with texture synthesis from a single image example, the target measure~$\nu$ is an empirical feature distribution with finite support 
$$\nu=\sum_{j=1}^m\delta_{F_j(\vv)},$$
composed of $m$ features $y_j= F_j(\vv)$.
In this case, the semi-dual formulation of optimal transport \eqref{eq:minimax} simplifies to
\begin{equation}\label{eq:semidualot}
\mathrm{OT}_c(\mu_\theta,\nu) = \max_{\psi \in \R^m} \FF(\theta,\psi) = \int \psi^{c}(x) d \mu_\theta(x) + \frac{1}{m}\sum_{j=1}^m \psi_j,
\end{equation}
where $\psi_j=\psi(y_j)$ and where the $c$-transform of $\psi$ writes $\psi^{c}(x) = \min_{j} \left[ c(x,F_j(\vv)) - \psi_j\right]$.
The main interest of this formulation is that it involves only a finite-dimensional vector $\psi \in \R^m$, which can be numerically optimized.
Notice also that the computation of the $c$-transform $\psi^c(x)$ boils down to a biased nearest-neighbor assignment in the feature space.
Combining the texture formation model $\mu_\theta$ introduced in~\eqref{eq:def_mu} and the functional of interest $\FF$ defined in~\eqref{eq:semidualot}, the final problem to optimize reads
\begin{equation}\label{eq:expectation}
\min_\theta\mathrm{OT}_c(\mu_\theta,\nu) =\min_\theta\max_{\psi \in \R^m} \FF(\theta,\psi) = \E_{Z\sim\zeta}\left[ \frac{1}{n}\sum_{i=1}^n \psi^c(F_i \circ g_\theta(Z)) + \frac{1}{m}\sum_{j=1}^m \psi_j \right],
\end{equation}
where $\theta$ can be a single image (Section \ref{sec:image-optim}) or the parameters of a generative network with input noise $Z$ (Section \ref{sec:generator}). 
The end of this section is focused on the computation of the gradients of this quantity with respect to $\theta$ and $\psi$.

\subsection{Gradient with respect to $\theta$} 

This section discusses the computation of the gradient with respect to the parameter~$\theta$ of the optimal transport cost $\mathrm{OT}_c(\mu_\theta,\nu)$. 
We provide a sufficient condition (related to the regularity of the generator $g_{\theta}$ and to the feature distribution $\mu_{\theta}$) that ensures existence of $\nabla_{\theta} \FF(\theta,\psi)$.
Besides, under the same condition, we show that $\nabla_\theta \mathrm{OT}_c(\mu_\theta,\nu) = \nabla_\theta \FF(\theta,\psi^*)$ with $\psi^*\in\argmax_\psi \FF(\theta,\psi)$ as soon as both terms are well defined.\\

As can be observed in expression \eqref{eq:expectation}, the computation of the gradient of $\FF$ with respect to $\theta$  only involves the differentiation of $\psi^c(F_i \circ g_\theta(Z))$.  
In order to examine the regularity of $\psi^c$, we introduce the open Laguerre cells
\begin{equation}
L_i(\psi) = \left\{ x ~|~\forall k\neq i, ~c(x,F_i(\vv)) - \psi_i < c(x,F_k(\vv)) - \psi_k \right\}.
\end{equation}
A simple but crucial remark directly follows from the definition of the $c$-transform: in the Laguerre cell $L_i(\psi)$, the $c$-transform expresses as
\begin{equation}
\forall x \in L_i(\psi), \quad 
\psi^c(x) = c(x,F_i(\vv)) - \psi_j .
\end{equation}
Therefore, $\psi^c$ inherits the regularity of $c$ in the Laguerre cells, and thus, when $c$ is smooth, differentiability problems can only appear at the boundaries of the Laguerre cells.
In order to avoid such singularities, we formulate the following hypothesis that constrains the feature distribution of the texture model.

\begin{hyp}\label{hyp:laguerre}
$g$ satisfies Hypothesis~\ref{hyp:laguerre} at $(\theta, \psi)$ if 
$\zeta\left((F_i\circ g_\theta)^{-1}\{ \cup_j L_j(\psi)  \}\right)=1$ for any position $i$, 
that is, for a given variable $\theta$, all the generated local features are almost surely within the Laguerre cells defined by $\psi$.
\end{hyp}

For example, if $c(x,y)=\|x-y\|_p^p$ with $p>1$, then $\psi^c$ is smooth on $\cup_j L_j(\psi)$, whose complement is negligible for the Lebesgue measure.
Indeed, in this case its complement is given by the union of the sets
\begin{equation}
B_{jk}(\psi) = \left\{ x ~|~ \|x - F_j(\vv))\|_p^p - \psi_j = \|x-F_k(\vv)\|_p^p - \psi_k \right\}
\quad (1 \leq j,k \leq m)
\end{equation}
which has zero Lebesgue measure, because each $B_{jk}$ is contained in a sub-manifold of dimension lower than $p$.
Therefore, Hypothesis~\ref{hyp:laguerre} is satisfied if, for any~$i$, $(F_i\circ g_\theta)\sharp\zeta$ is absolutely continuous with respect to the Lebesgue measure.\\

We also introduce a regularity hypothesis for the generative model $g_\theta$ which will allow us to differentiate under the expectation.
\begin{hyp}\label{hyp:glip} There exists $K:\Theta\times \calZ \to \R_+$ such that for all $\theta$, there exists a neighborhood $V$ of $\theta$ such that $\forall\theta' \in V$ and for $\zeta$-almost every $z$,
\begin{equation}
\|g(\theta, z) - g(\theta',z)\| \leq K(\theta,z)\| \theta-\theta'\|
\end{equation}
with $K$ verifying for all $\theta$,
$\mathbf{E}_{Z\sim\zeta}\left[K(\theta,Z)\right]< \infty$.
\end{hyp}

We can now express the gradient of $\mathcal{J}$ with respect to the parameter $\theta$.

\begin{theorem}\label{thm:dF}
Assume $c$ to be $\mathscr{C}^1$ and assume that the features $F_i$ are all differentiable and Lipschitz. 
Let $g$ satisfy Hypothesis~\ref{hyp:glip}.
Let $\theta_0$ be a point where $\theta\to g_\theta(z)$ is differentiable $\zeta(dz)$-almost surely and let $g$ satisfy Hypothesis~\ref{hyp:laguerre} at $(\theta_0,\psi)$. 
Then $\theta \mapsto \FF(\theta,\psi) $ is differentiable at $\theta_0$ and
\begin{equation} \label{eq:gradF}
\nabla_\theta \FF(\theta_0,\psi) =\frac{1}{n}\sum_{i=1}^n 
\E_{Z\sim\zeta}\left[ \partial_\theta (F_i \circ g)(\theta_0,Z)^T \nabla\psi^c(F_i \circ g(\theta_0,Z))  \right]
\end{equation}
with 
$
\nabla\psi^c(F_i\circ g(\theta_0,z)) = \nabla_x c(F_i\circ g(\theta_0,z), F_{\sigma(i)} \vv)
$

where $\sigma(i)$ is the unique index such that $F_i\circ g(\theta_0,z) \in L_{\sigma(i)}(\psi)$
(which exists $\zeta(dz)$-almost surely).
\end{theorem}
\begin{proof}
From expression~\eqref{eq:expectation}, we see that the proof consists in differentiating the function
\begin{equation}\label{eq:Hi}
H_i(\theta,\psi) = \E\left[h_i(\theta,\psi,Z)\right] 
, \quad \text{with} \quad 
h_i(\theta,\psi,Z) = \psi^c(F_i \circ g_{\theta}(Z)) .
\end{equation}
Thanks to Hypothesis~\ref{hyp:laguerre}, for $\zeta$-almost all $z$, there exists an index $j$ such that ${F_i \circ g_{\theta}(z) \in L_j(\psi)}$ and thus $L_j(\psi)$ is an open neighborhood of $F_i \circ g_{\theta}(Z)$ where $\psi^c$ is differentiable.
Using the chain rule, we  get that for $\zeta$-almost all $z$, $\theta \mapsto h_i(\theta, \psi, z)$ is differentiable at $\theta_0$ and
\begin{equation}
\nabla_{\theta} h_i(\theta, \psi, z) = \partial_\theta (F_i \circ g)(\theta_0,Z)^T \nabla\psi^c(F_i \circ g(\theta_0,Z)) .
\end{equation}
In order to differentiate under the expectation in \eqref{eq:Hi}, we  have to get an integrable bound on the finite differences of $h_i$.
For that, let us denote by $\kappa_c$ the Lipschitz constant of $c$ and $\kappa_i$ the Lipschitz constant of $F_i$.
Let us recall from~\cite{santambrogio2015ot} that $\psi^c$ is also $\kappa_c$-Lipschitz.
Therefore, from Hypothesis~\ref{hyp:glip}, we get a neighborhood $V$ of $\theta_0$ such that for any $\theta\in V$ and for $\zeta$-almost all $z$,
\begin{equation}
|h_i(\theta,\psi,z)) - h_i(\theta_0, \psi,z)| \leq \kappa_c \kappa_i \|g(\theta,z) - g(\theta_0,z)\|
\leq \kappa_c \kappa_i K(\theta,z) \|\theta - \theta_0 \| ,
\end{equation}
with $\E[K(\theta,Z)] < \infty$. This bound allows us to differentiate under the expectation and to get the expression
\begin{equation}
\nabla_{\theta} H_i(\theta_0,\psi) = \E[ \partial_{\theta}(F_i \circ g)(\theta_0,Z)^T \nabla \psi^c(F_i \circ g(\theta,Z)) ] .
\end{equation}
Gathering the terms for all $i$ leads to the desired result.
\qed
\end{proof}

Finally, upon existence, we can relate the gradient of $\FF$ to the gradient of the optimal transport.
\begin{theorem}\label{thm:2}
Let $\theta_0$ such that $\theta \mapsto  \OTc(\mu_\theta, \nu)$  and $\theta \mapsto   \FF(\theta,\psi^*) $ are differentiable at~$\theta_0$ with $\psi^*\in \argmax_\psi \FF(\theta_0,\psi)$ then
\begin{equation}\label{eq:thm}
\nabla_\theta\OTc(\mu_{\theta_0}, \nu)  = \nabla_\theta \FF(\theta_0,\psi^*)
\end{equation}
\end{theorem}
\begin{proof}Let us fix $\psi^*\in \argmax_\psi \FF(\theta_0,\psi)$. 
The function $H(\theta) = \FF(\theta_0,\psi^*) - \OTc(\mu_\theta, \nu)$ is differentiable at $\theta_0$ and maximal at $\theta_0$. Therefore we get $\nabla_\theta H(\theta_0) = 0$.
\qed
\end{proof}

The gradient expression found here will be later used to minimize $\theta\to\OTc(\mu_\theta, \nu)$. 
A stochastic gradient-based algorithm will be used to reach a local minimum of this optimal transport cost and learn the texture model~$g_\theta(Z)$. 
Notice that evaluating $\nabla_\theta \FF(\theta_0,\psi^*)$ in \eqref{eq:thm} requires the knowledge of an optimal potential $\psi^*(\theta_0) \in \argmax_\psi \FF(\theta_0,\psi)$. 
The next section discusses  how to approximate such an optimal potential.

\subsection{Super-gradient with respect to $\psi$}

The computation of the exact transport cost is a challenging task and it has been widely studied in the literature. Recently, a stochastic method for approximating the optimal dual potential for the  semi-discrete case has been studied~\cite{genevay2016ot}. We propose to use this approach to approximate the optimal potential $\psi^*$ with a stochastic gradient ascent scheme.  
Hereafter we recall with proofs some known facts about the concavity and the super-gradients of $\mathcal{J}$, which will be used in the stochastic optimization algorithm.

Let $\theta$ be a fixed parameter, in order to improve the readability we set $G_i = F_i\circ g_\theta$ and remove all the $\theta$ dependencies. The maximization problem we aim at solving writes
\begin{equation}
\max_{\psi\in\mathbf{R}^m} \mathcal{J}(\psi) = 
\mathbf{E}_{Z\sim\zeta}\left[J(\psi, Z) \right],
\textrm{ where } 
J(\psi, z) = \frac{1}{n}\sum_{i=1}^n \psi^c(G_i(z)) + \frac{1}{m}\sum_{j=1}^m \psi_j.
\end{equation}
We first recall the following result from the optimal transport theory.
\begin{theorem} 

(i) For any $z\in\mathcal{Z}$, the function $\psi \to J(\psi,z)$ is concave on $\R^m$. 

(ii) The function $\psi \to \mathcal{J}(\psi)$ is concave on $\R^m$.

\end{theorem}

\begin{proof}
Let $\psi^1\in\mathbf{R}^m$, $\psi^2\in\mathbf{R}^m$ and $t\in[0,1]$ and fix $z\in\mathcal{Z}$. 
Recalling  the $c$-transform definition  $\psi^{c}(x) = \min_{j} \left[ c(x,F_j(\vv)) - \psi_j\right]$, we have
\begin{align}
(t\psi^1+(1-t)\psi^2,z) 
=& \frac{1}{n}\sum_{i=1}^n\min_{j}\left[c(G_i(z), F_j(u_0)) - t\psi^1_j - (1-t)\psi^2_j\right] \nonumber \\
&+\frac{1}{m}\sum_{j=1}^m t\psi^1_j +(1-t)\psi^2_j\\
= &\frac{1}{n}\sum_{i=1}^n\left[c(G_i(z), F_{j^*(i)}(u_0)) - t\psi^1_{j^*(i)} - (1-t)\psi^2_{j^*(i)}\right] \nonumber \\
&+\frac{1}{m}\sum_{j=1}^m t\psi^1_j +(1-t)\psi^2_j,
\end{align}
where $j^*(i) \in \argmin_{j}\left[c(G_i(z), F_j(u_0)) - t\psi^1_j - (1-t)\psi^2_j\right]$. 
Splitting $c(\cdot,\cdot)=tc(\cdot,\cdot)+(1-t)c(\cdot,\cdot)$ and using the property of the minimum function, we get that
\begin{align}
J(t\psi^1+(1-t)\psi^2,z) 
&\geq \frac{1}{n}\sum_{i=1}^nt\min_{j}\left[c(G_i(z), F_j(u_0)) - \psi^1_j\right]\\
&+ \frac{1}{n}\sum_{i=1}^n (1-t)\min_{j}\left[c(G_i(z), F_j(u_0)) - \psi^2_j\right] \\
&+\frac{1}{m}\sum_{j=1}^m t\psi^1_j +(1-t)\psi^2_j\\
&= tJ(\psi^1,z) + (1-t)J(\psi^2,z),
\end{align}
which proves the first point.
The second point follows by taking the expectation of both sides.
\qed
\end{proof}

We can now state the following result that gives a super-gradient for $\mathcal{J}$.

\begin{theorem}\label{thm:subgradient}
Let us denote by $(e_j)_{1 \leq j \leq m}$ the canonical basis of $\mathbf{R}^m$.
Let $z\in\mathcal{Z}$, and $\psi\in\mathbf{R}^m$.
Then a super-gradient of $J(\cdot, z)$ at point $\psi$ is given by
\begin{equation}\label{eq:supergrad}
D(\psi, z) = \frac{1}{m}\mathbf{1}_m - \frac{1}{n}\sum_{i=1}^n\mathbf{e}_{j^*(G_i(z),\psi)},
\end{equation}
where 
\begin{equation}
j^*(G_i(z),\psi) \in \argmin_j\left(c(G_i(z), F_j(v)) - \psi_j\right).
\end{equation}
It follows that
\begin{equation}
\mathcal{D}(\psi) = \E_{Z\sim\zeta}\left[D(\psi, Z)\right], 
\end{equation}
is a super-gradient of $\mathcal{J}$ at point $\psi$.
\end{theorem}

\begin{proof}
Take $z\in\mathcal{Z}$ and $\psi_0\in\mathbf{R}^m$. In order to demonstrate that $D(\psi,z)$ is a super-gradient, we need to show that
 \begin{equation}
\forall \psi' \in \R^m, \quad J(\psi', z) \leq J(\psi, z) + \langle D(\psi,z), \psi' - \psi \rangle.
 \end{equation}
 We have
 \begin{align}
J(\psi, z) -  J(\psi', z)  + \langle D(\psi,z), \psi' - \psi \rangle &= \frac{1}{n}\sum_{i=1}^n \min_j\left[c(G_i(z),F_j(u_0)) - \psi_j\right]\label{eq:firstterm}\\
&- \frac{1}{n}\sum_{i=1}^n \min_j\left[c(G_i(z),F_j(u_0)) - \psi'_j\right]\label{eq:secondterm}\\
&- \frac{1}{n}\sum_{i=1}^n \left( \psi'_{j^*(G_i(z),\psi)} - \psi_{j^*(G_i(z),\psi)} \right).
 \end{align}
Then, $j^*(G_i(z),\psi)$ satisfies the min in \eqref{eq:firstterm} and the min in the second term \eqref{eq:secondterm} is by definition smaller than the value taken at $j^*(G_i(z),\psi)$. Therefore all terms compensate and we have
 \begin{align}
J(\psi, z) -  J(\psi', z)  + \langle D(\psi,z), \psi' - \psi \rangle 
\geq 0.
 \end{align}
 \qed
\end{proof}

In order to approximate an optimal potential, we will rely on an averaged stochastic super-gradient ascent as proposed in~\cite{genevay2016ot}. More precisely, at each step $k$ we sample $z\sim\zeta$ and we update $\psi^k$ as follows
\begin{equation}
\psi^k = \psi^{k-1} + \frac{1}{\sqrt{k}}D(\psi^{k-1},z)
\end{equation}
The final estimate for $\psi^*$ is  obtained by averaging the estimates from a given point $k_0$ (set to 1 in experiments)
\begin{equation}
\hat{\psi}^k = \frac{1}{k-k_0+1} \sum_{\ell=k_0}^{k} \psi^{\ell}.
\end{equation}

Finally, the combination of Theorem \ref{thm:2} and Theorem \ref{thm:subgradient} provides us a way to approximate the gradient of the optimal transport cost $\OTc(\mu_\theta, \nu)$ with respect to $\theta$. Then, we propose to minimize this optimal transport cost by performing a stochastic gradient descent algorithm. In practice we will use the Adam algorithm~\cite{kingma2014adam}. We detail the proposed algorithm and its application to texture synthesis in the next section.



\section{Combining Several Feature Distributions}
\label{sec:feat}

In the previous section, we defined the problem and stated theoretical results for a single collection of local features. 
In practice, texture synthesis requires to simultaneously enforce different collections of features. 
For instance one can use local features at different scales in order to model image patterns of various sizes~\cite{Kwatra,galerne2018texture,shaham2019singan}. 
We now present a general framework able to combine several sets of features.

We therefore consider a set of $\LL$ different features given by $F^l = (F_i^l)_{1 \leq i \leq n_l}$, $l=1, \ldots, \LL$, each one of size ${n_l}$, and we propose to minimize the quantity 
\begin{equation}\label{eq:multiscale}
\mathcal{L}_{\text{GOTEX}}(\theta) = \sum_{l=1}^{\LL} \OTc(\mu_\theta^l, \nu^l)=\sum_{l=1}^{\LL} \max_{\psi^l\in\mathbb{R}^{m_l}}\FF^l(\theta, \psi^l),
\end{equation}
where $\mu_\theta^l$ (resp. $\nu^l$) is the  distribution relative to the feature $l$ of the synthesis (resp. of the example $\vv$ that contains $m_l$ features), i.e. 
$\mu_\theta^l = \frac{1}{n_l}\sum_{i=1}^{n_l}  \delta_{F^l_i(\theta)}$,  
and where 
$$\FF^l(\theta, \psi^l) = \E_{Z\sim\zeta}\left[ \frac{1}{n_l}\sum_{i=1}^{n_l} \psi^{l,c}(F^l_i \circ g_\theta(Z)) + \frac{1}{m_l}\sum_{j=1}^{m_l} \psi^l_j \right],$$
with the $c$-transform defined as $\psi^{l,c}(x) = \min_{1\leq j\leq m_l} \left[ c(x,F^l_j(\vv)) - \psi^l_j\right]$.\\

Upon existence of all terms, the gradient of the quantity \eqref{eq:multiscale} then reads
\begin{align}\label{eq:gradL}
\nabla_\theta \mathcal{L}(\theta) &
 =\sum_{l=1}^{\LL} \nabla_\theta \FF^l(\theta, \psi^{l,*}),
\end{align}
where $\psi^{l,*}$ is an optimal Kantorovich potential for $\OTc(\mu_\theta^l, \nu^l)$. 
The potential $\psi^{l,*}$ can be approximated with a super-gradient ascent (using the super-gradient of $\mathcal{J}^l$ obtained in Theorem \ref{thm:subgradient}). 
Therefore, in order to minimize the loss defined in  \eqref{eq:multiscale}, we propose to repeat the two following steps: 
\begin{enumerate}
\item for each $l$ compute $\hat{\psi}^l$ that approximates $\psi^{l,*}$ with an averaged stochastic super-gradient ascent;
\item perform an optimization step with respect to $\theta$ using \eqref{eq:gradL} with $\hat{\psi^l}$. In the experiments we use the L-BFGS~\cite{lbfgs} for image-based optimization in section~\ref{sec:image-optim} and the Adam algorithm~\cite{kingma2014adam} for neural-network training in section~\ref{sec:generator}.
\end{enumerate} 
The multi-feature process is summarized in Algorithm~\ref{alg:MStexgenGD}.
The notation $D(\psi,z)$ was defined in~\eqref{eq:supergrad} but is changed here in $D(\theta,\psi,z)$  to recall the dependency on the current $\theta$.
We now study different choices for the features $(F^l)$ for texture synthesis.

\begin{algorithm}[ht!]\small
   \caption{Texture Synthesis with 
   Prescription of Several Feature Distributions}
   \label{alg:MStexgenGD}
\begin{algorithmic}[1]
   \State {\bfseries Input:} target image $\vv$, initial parameter $\theta_0$, learning rates $\eta_\theta$ and $\eta_\psi$, number of iterations $N_{u}$ and $N_{\psi}$, number of features $\LL$, {optimizer \emph{Optim} (s.t. Adam or L-BFGS)}
   \State {\bfseries Output:} learned parameter $\theta^*$
   \State $\theta \leftarrow \theta_0$ and $\psi^{l,0}=0 \text{ for } l=1\ldots \LL$
   \For{$k=1$ {\bfseries to} $N_{\theta}$}
   \For{$l=1$ {\bfseries to} $\LL$}
   \State $\tilde \psi^{l,0}\leftarrow \psi^{l,k-1}$ 
   
   \For{$\ell=1$ {\bfseries to} $N_{\psi}$}
   
   \State Draw a sample $z \sim \zeta$.
   \State$\tilde \psi^{l,\ell}=\tilde \psi^{l,\ell-1}+\eta_\psi D(\theta^{k-1}, \tilde \psi^{l,\ell-1}, z) $

   \Comment{Gradient ascent on $\psi^{l,k}$ (Theorem \ref{thm:subgradient})}
   \EndFor
   
   \State $\psi^{l,k}\leftarrow\tilde \psi^{l,N_{\psi} }$

   \EndFor
   \State $\theta^{k} \leftarrow 
   {Optim}(\sum_{l=1}^{\LL} \nabla_\theta \FF^l(\theta^{k-1}, \psi^{l,k}),\eta_\theta)$\Comment{Gradient-based update of $\theta$ (Theorem \ref{thm:dF})} 
   \EndFor
\end{algorithmic}
\end{algorithm}

\subsection{Gaussian pyramid of patches}\label{ssec:pyramid}

The simplest local features we can define are patches. Patches are small square sub-images of size $s\times s$, and we define the $i$th patch extractor $P_i$ as the linear operation that extracts the $s\times s$ pixels around the pixel $i$ of an image. 
In this paragraph, the feature operators $F_i$ are defined from such patch extractors operating at different scales. 

As previously mentioned, prescribing the distribution of patches of different sizes is necessary to model image patterns of different scales~\cite{Kwatra,galerne2018texture}. 
Then, in order to construct a multi-scale collection of patches, we create a pyramid of down-sampled and blurred images. For each scale $l = 1, \ldots, \LL$, we use a linear blurring and down-sampling operator~$G_s$ that computes, for an image $u$ with $n = n_W \times n_H$ pixels,  a reduced version $u_l = G_l u$ with $n_l = \frac{n_H}{2^{l-1}} \times \frac{n_W}{2^{l-1}}$ pixels.
We then define by $P_i^l$ the operator that extracts the $i$th patch of the blurred and down-sampled image $u_l$, that is $P_i^l = P_i\circ G_l$. In~\cite{Houdard_ssvm21} we demonstrated that state-of-the art texture synthesis results can be obtained with the collection of features $\{P_i^l\}_{i,l}$.

In the following, when using such patch features within our optimal transport framework to estimate an image $\theta$, we consider the feature index $l$ for scales and the loss we aim at minimizing writes
\begin{equation}
\label{eq:OTpatchloss} 
\mathcal{L}_{\text{patch}}(\theta) = \sum_{l=1}^{\LL} \OTc(\mu_\text{pat}^l(\theta), \nu_\text{pat}^l)
, \quad
\text{where} \quad
\mu_\text{pat}^l(\theta) = \frac{1}{n_l}\sum_{i=1}^{n_l} (P_i\circ G_l \circ g_\theta)\sharp\zeta
\end{equation}
and where similarly $\nu_\text{pat}^l$ is the blurred and down-sampled  patch distribution of the example image example $\vv$.
The loss \eqref{eq:OTpatchloss}  is referred to as the \textbf{GOTEX-patch} loss.

\subsection{VGG features}\label{sec:VGG}
The use of \emph{deep features} extracted from a deep neural network has proven to be successful for texture generation. In~\cite{gatys_texture_2015}, the authors proposed to use the outputs from different layers of the pre-trained VGG-19 network from~\cite{VGG}. This fully convolutional network was introduced by the Visual Geometry Group (VGG) from the University of Oxford. It consists of a sequence of sixteen convolutional layers of kernel size $3\times3$ with five $2\times2$ max-pooling layers with stride 2,  and  three final fully connected layers (see~\cite{VGG} for the detailed architecture).

In this work, we are using the VGG-19 network with weights that were pre-trained for an image classification task on the ImageNet dataset~\cite{deng2009imagenet}. We also replaced the max-pooling layers with average-pooling layers as done in~\cite{gatys_texture_2015} in order to reduce checkerboard artifacts. 
For the features, we consider the outputs from five layers ($\LL = 5$): the first convolutional layer and then the first convolutional layer after each pooling layer (\textrm{pool1} to \textrm{pool4}). The related loss to optimize reads
\begin{equation}
\label{eq:OTVGGloss} 
\mathcal{L}_{\text{VGG}}(\theta) = \sum_{l=1}^{\LL} \OTc(\mu_\text{VGG}^l(\theta), \nu^l_\text{VGG})
, \quad 
\text{with} \quad
\mu_\text{VGG}^l(\theta) = \frac{1}{n_l}\sum_{i=1}^{n_l} (F_i^l \circ g_t)\sharp\zeta
\end{equation}
and where $F_i^l$ corresponds to the normalized\footnote{Features maps are normalized by the tensor dimension, that is spatial and channels size.} $i$th output from the layer $l$ of the VGG-19 network and $n_l$ the spatial size of this layer, and where similarly $\nu^l_\text{VGG}$ is the feature distribution of the example at layer $l$.
The loss~\eqref{eq:OTVGGloss} will be referred to as the \textbf{GOTEX-VGG} loss.

\subsection{Mixing features}\label{sec:mixed}

The major flaw from VGG-19 feature decomposition (see Section~\ref{ssec:comp_VGG}) is that visual artifacts such as checkerboard patterns or color inconsistencies can appear on the synthesized textures. Generally speaking, these problems are solved with a posterior histogram equalization or with the use of a median filter on the final result. On the other hand, with the multiscale patch decomposition (Section~\ref{ssec:pyramid}) one may fail to recover thin details at larger scales (due to the blurring in the Gaussian pyramid). 
Conversely, patches at the image scale accurately restore both the local details and the color consistency of the image, while VGG features from deeper layers can represent well larger patterns of the image.

Since several distributions of features can be combined with the GOTEX model, we propose to both enforce the VGG-19 features from deep layers -- which represent large structures and global geometry from the target texture -- and patches features from the firsts scales of the image -- which represent local details and color distributions from the target texture. 
We therefore propose to optimize the following loss
\begin{equation}
\label{eq:OTmixedloss} 
\mathcal{L}_{\text{\mixed}}(\theta) =
\lambda \OTc(\mu_\text{pat}^1(\theta), \nu^1_\text{pat})
+
\sum_{l=2}^{\LL} \OTc(\mu_\text{VGG}^l(\theta), \nu^l_\text{VGG})
\end{equation}
where the scalar $\lambda\geq 0$ is used to balanced the discrepancy of the range values between patches and VGG features. The loss~\eqref{eq:OTmixedloss} with $\lambda=1$ will be referred to as the \textbf{GOTEX-\mixed} loss.

For the sake of completeness, we also consider a loss that includes patch and VGG features at all scales:
\begin{equation}
\label{eq:OTallloss} 
\mathcal{L}_{\text{all}}(\theta) =
\lambda \sum_{l=1}^{\LL} \OTc(\mu_\text{pat}^l(\theta), \nu^l_\text{pat})
+
\sum_{l=1}^{\LL} \OTc(\mu_\text{VGG}^l(\theta), \nu^l_\text{VGG})
\end{equation}
For $\lambda = 1$, this will be referred to as the \textbf{GOTEX-all} loss.

\subsection{Texture barycenters}\label{sec:bary}

In the framework of Gatys et al.~\cite{gatys_texture_2015}, the Gram loss does not represent a distance between distributions. Therefore, it does not provide relevant results when dealing with texture interpolation. 
This issue has already been pointed out for instance in~\cite{vacher2020texture}. 
In contrast, when $c(x,y)= \|x-y\|^p$, the optimal transport cost is related to a true distance
\begin{equation}
W_p(\mu,\nu) = \OTc(\mu,\nu)^\frac{1}{p}
\end{equation}
called the $p$-Wasserstein distance, which allows to define relevant texture interpolation paths.
Indeed, restricting to the cost $c(x,y)=||x-y||^2$, we will now show the connection between $2$-Wasserstein barycenters and the interpolation of $K$ textures $u_0, \ldots, u_{K-1}$.
We recall that $\mathcal{L}_\text{GOTEX}(\theta, u_k)$ is the loss function~\eqref{eq:multiscale} related to the exemplar texture~$u_k$.
This loss depends on the feature distribution $(\nu_k^l)_{1 \leq l \leq N_F}$ extracted from $u_k$, for any set of features described in the last paragraphs.
Given also weights $\alpha_0, \ldots, \alpha_{K-1} \geq 0$ such that $\alpha_0 + \ldots + \alpha_{K-1} = 1$, we propose to define a Wasserstein barycenter of the textures $(u_0, \ldots, u_{K-1})$ with weights $(\alpha_1, \ldots, \alpha_K)$ as the texture distribution $g_{\theta} \sharp \zeta$ where~$\theta$ minimizes the quantity
\begin{equation}\label{eq:bar_loss}
\sum_{k=0}^{K-1} \alpha_k \mathcal{L}_\text{GOTEX}(\theta, u_k)
= \sum_{k=0}^{K-1}  \sum_{l=1}^{N_F} \alpha_k W_2^2(\mu_{\theta}^l, \nu_k^l)
\end{equation}
In comparison, for each separate feature $l$, the usual Wasserstein barycenter of $\nu_1^l, \ldots, \nu_K^l$ defined in~\cite{agueh2011barycenters} is a solution of
\begin{equation}
 \argmin_{\mu^l} \sum_{k=0}^{K-1} \alpha_k W_2^2(\mu^l, \nu_k^l) .
\end{equation}
Our definition of Wasserstein texture barycenter is thus an adaptation  of this notion with two modifications. First, for each feature $l$, $\mu_{\theta}^l$ is constrained to be the distribution of the feature $l$ of $g_{\theta} \sharp \zeta$. Second, all features $l$ are treated simultaneously.
The barycentric loss~\eqref{eq:bar_loss} thus provides an interpolation path between textures $u_0, \ldots, u_{K-1}$ while belonging to a set of texture models constrained by the choice of $g_{\theta} \sharp \zeta$.

%

\section{Single Image GOTEX}\label{sec:image-optim}

In this section we focus on the particular case of synthesizing a single image. This corresponds to the case where the model just generates a single image $\theta$ by taking $g_\theta(z) = \theta-z$ for all $z$ and $\zeta = \delta_0$ (see remark~\ref{rem:si}). This amounts to minimize w.r.t. the image $\theta$ the optimal transport cost between its discrete feature distribution $\mu_\theta$ and a discrete target feature distribution $\nu$. 
This yields a pixelwise optimization algorithm that  minimizes a fully discrete optimal transport cost between patch distributions.

We present in section \ref{sec:SIGOTEX} the single image setting of our GOTEX framework. We provide experimental results using the features described in section \ref{sec:feat} and discuss their pros and cons. 
Related texture synthesis methods are briefly reviewed in section \ref{sec:previous_work_OT}. In particular, we show that using image patches as features in this single image setting gives an elegant interpretation of the algorithm of~\cite{Gutierrez_ssvm2017} which combines the patch-based optimization framework from~\cite{Kwatra} with optimal transport. 

Our Single Image GOTEX approach, using either patch or deep features, is  compared in section \ref{sec:exp_image_optim} to the previously introduced state-of-the-art methods. 
We illustrate that the proposed  hybrid losses \eqref{eq:OTmixedloss} and  \eqref{eq:OTallloss}  allows to capture higher order statistics than the perceptual loss proposed in~\cite{gatys_texture_2015}. Finally, we show in section \ref{ssec:other_app} that our GOTEX framework is well-suited for other tasks related to texture imaging such as texture barycenters and texture inpainting.

\subsection{Single texture synthesis}\label{sec:SIGOTEX}

In the case of a single image generation, the GOTEX framework amounts to minimize a discrete optimal transport cost. Let $\theta \in \R^n$  be the image to synthesize with $n$ pixels and $\mu_\theta= \frac{1}{n}\sum_{i=1}^n \delta_{F_i(\theta)}$ its discrete patch distribution. In order to impose on the image~$\theta$ the patch distribution $\nu=\frac{1}{m}\sum_{j=1}^m \delta_{F_j(u_0)}$ of the exemplar image $\vv$, we aim at solving problem \eqref{eq:expectation} that here simplifies to 
\begin{equation}
\min_{\theta \in \R^n} \mathrm{OT}_c\left(\mu_\theta, \nu \right) = \min_{\theta \in \R^n} \max_{\psi \in \R^m} \FF(\theta,\psi). \label{eq:discretepb}
\end{equation}
In this particular discrete case, the functional of interest writes
\begin{equation}\label{eq:FF}
\FF(\theta,\psi) = \frac{1}{n} \sum_{i=1}^n \psi^c(F_i (\theta) )+   \frac{1}{m}\sum_{j=1}^m\psi_j,
\end{equation}
with
\begin{equation}
\psi^c(F_i(\theta)) = \min_j\left[c(F_i(\theta), F_j(u_0)) -\psi_j\right].
\end{equation}
The minimization of \eqref{eq:discretepb} is then done with the proposed algorithm \ref{alg:MStexgenGD} which is in this case \textbf{not} stochastic. From now on and for all the experiments, we set the cost $c$ to be the quadratic cost and in the following we illustrate the versatility of our framework with synthesis experiments on various textures and features.
After introducing the experimental set-up in section \ref{ssec:numerical}, results obtained from different feature choices (exposed in section~\ref{sec:feat}) are discussed in section \ref{ssec:comp_loss}. In section \ref{ssec:stability} we finally discuss the robustness of the method with respect to the choice of the initial image.

\subsubsection{Experimental setting}\label{ssec:numerical}

In this section, the L-BFGS algorithm~\cite{lbfgs} is used as the optimizer in Algorithm \ref{alg:MStexgenGD}.
We resort to the PyTorch implementation with the default parameter setting (for instance, $\eta_\theta = 1$).
The initial image $\theta^0$ is randomly sampled from a Gaussian white noise of mean $m_\text{target}$ and variance $0.01\sigma_\text{target}$ where $m_\text{target}$ and $\sigma_\text{target}$ are respectively the mean and the variance of the target image.
The patch size used for the Gaussian pyramid of patches is $4\times 4$ pixels at each scale, and the number of scales is $N_F=4$.
We also use the KeOps library~\cite{keops} to define the cost matching kernel and compute the biased nearest neighbor map \eqref{eq:jstar} efficiently on GPU. 
The number of inner iterations $N_\psi$ is set to 10.
Example images are all of size $256\times256$ in order to compute VGG features appropriately\footnote{Note that a multi-scale approach is required for high resolution synthesis with VGG features as studied in~\cite{gatys2017controlling}.}. 
The VGG features are computed using the pre-trained VGG-19 network from the PyTorch library.

\subsubsection{Experimental results with various features}\label{ssec:comp_loss}

Here we present the single image synthesis results from our GOTEX framework using the losses GOTEX-patch \eqref{eq:OTpatchloss} , GOTEX-VGG~\eqref{eq:OTVGGloss}, GOTEX-\mixed~\eqref{eq:OTmixedloss} and GOTEX-all~\eqref{eq:OTallloss}. We present in Figure \ref{fig:compfeature} these results for $4$ given exemplar textures. Since we use the same framework either with patches or with VGG features, we are able to discuss precisely the pros and cons of each feature representation. The GOTEX-patch method (second column of Fig.~\ref{fig:compfeature}) shows good color consistency and respects the texture statistics at each scale. On the other hand, it tends to produce slightly smoother results. The GOTEX-VGG approach (third column of Fig.~\ref{fig:compfeature}) produces sharper results but at the cost of color inconsistencies and visual artifacts such as checkerboard artifacts~\cite{odena2016deconvolution}. These artifacts can be explained by the back-propagation through the non-linear VGG network 
and are known to appear with perceptual losses based on VGG, see e.g.~\cite{sajjadi2017enhancenet,johnson2016perceptual}. 
The over-smoothing effect when using patches is explained by the fact that the algorithm tends to make a compromise between the patches at each iteration (see section \ref{ssec:patches} for details on this point). Finally, by combining patches and VGG features (fourth column of Fig.~\ref{fig:compfeature}), the GOTEX-\mixed~method tackles both the color and artifact issues from the VGG features and the smoothing effect observed with GOTEX-patch.
The GOTEX-all method (fifth column of Fig.~\ref{fig:compfeature}) produces similar results but with long-range structures that are slightly better retrieved.

\begin{figure}[pt]
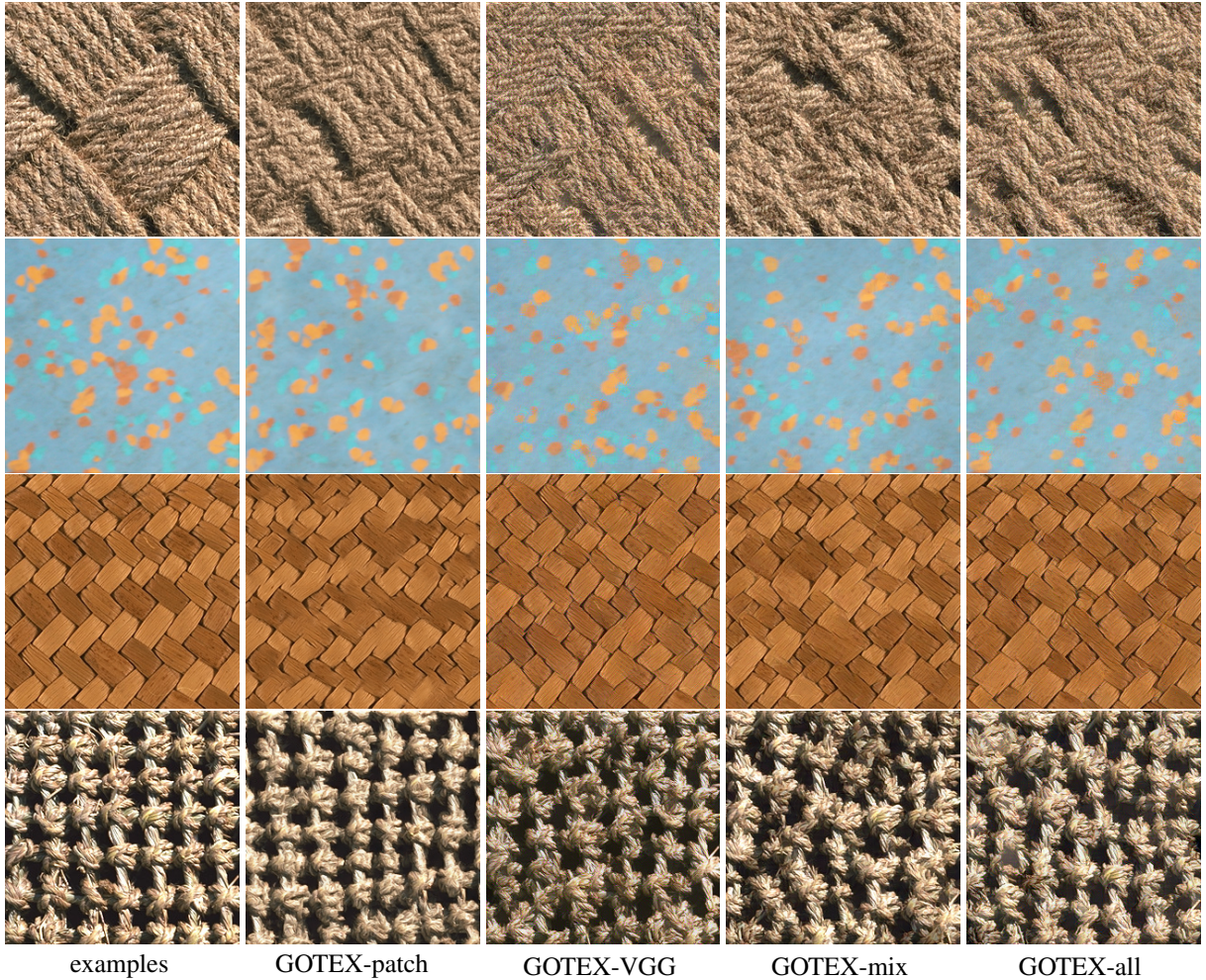

\centering
\newlength{\lfig}
\setlength{\lfig}{0.19\linewidth}

\foreach \x in {cordes_256, ground1013_small_256, raad1_256, raad8_256}
{\includegraphics[width=\lfig]{\x} %
\includegraphics[width=\lfig]{generated_\x_to_512} %
\includegraphics[width=\lfig]{\x_OTVGG_LBFGS_final} %
\includegraphics[width=\lfig]{\x_OTmixed_LBFGS_final} %
\includegraphics[width=\lfig]{\x_OTfullmixed_LBFGS_final}\\
}

\foreach \x in {examples, {\hspace*{-2mm}GOTEX-patch}, GOTEX-VGG, GOTEX-\mixed, GOTEX-all}
{\begin{minipage}{\lfig}
\centering
\x
\end{minipage} }

\caption{Results of  GOTEX (Alg. \ref{alg:MStexgenGD}) for image optimization using various combinations of features: 
 patches (\emph{GOTEX-patch}), 
 VGG (\emph{GOTEX-VGG}), 
mixing patches from higher scales with VGG from lower scales (\emph{GOTEX-\mixed}), 
and combining all features (\emph{GOTEX-all}).
}
\label{fig:compfeature}
\end{figure}

\subsubsection{Importance of the initialization}\label{ssec:stability}

When using the VGG features, the choice of the initial image $\theta^0$ has a strong impact on the final result. This can be explained by the fact that deeper layers in the VGG network only encode structures and do not encode colors. Therefore, the minimization of GOTEX-VGG is likely to find a texture that is a local minimum but has the wrong color distribution. This behavior was already documented in~\cite{gatys_texture_2015} where the authors proposed to perform an histogram matching at the end of the synthesis process. On the contrary, we show in Figure \ref{fig:stability} that using patches provides a very consistent synthesis method, not depending on the initial image $\theta^0$ .

\begin{figure}[t]
\setlength{\flen}{0.12\linewidth}
\centering
\setlength{\tabcolsep}{1pt}
\begin{tabular}{ccccccc}
\raisebox{.5\flen}{\includegraphics[width=\flen]{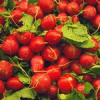}}
&\includegraphics[width=2\flen]{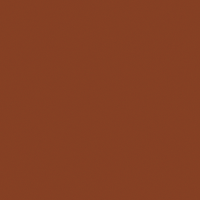}%
&\includegraphics[width=2\flen]{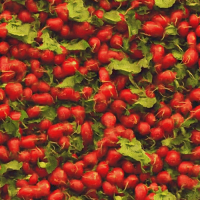}%
& \includegraphics[height=2\flen]{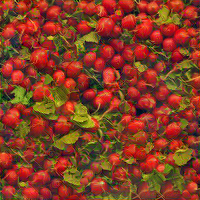}
\\
\raisebox{.5\flen}{\includegraphics[width=\flen]{radishes.png}}
&\includegraphics[width=2\flen]{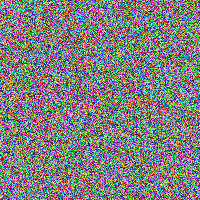}
&\includegraphics[width=2\flen]{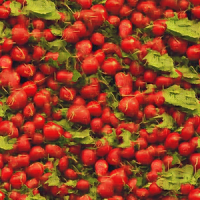}%
& \includegraphics[height=2\flen]{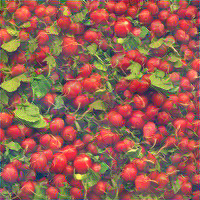}
\\
\raisebox{.5\flen}{\includegraphics[width=\flen]{radishes.png}}
&\includegraphics[width=2\flen]{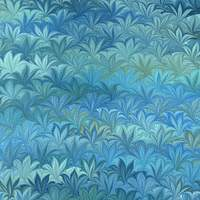}
&\includegraphics[width=2\flen]{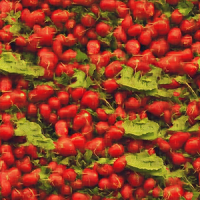}%
& \includegraphics[height=2\flen]{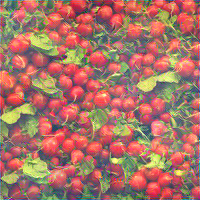}

\\
  \footnotesize a. Sample %
& \footnotesize b. Initialization
& \footnotesize c. GOTEX-patch
& \footnotesize d. GOTEX-VGG
\\
\end{tabular}
\caption{Both GOTEX-patch and GOTEX-VGG 
are run for the same $100\times100$ sample (a) with three initial images (b). GOTEX-path produces faithful $200\times 200$ synthesis (c) for any initialization whereas the GOTEX-VGG results (d) tends to produce color inconsistencies and artifacts when the color palette of the initial guess is not close enough to the target image.
\label{fig:stability}}
\end{figure}

\subsection{Related works on optimal patch transport}

\label{sec:previous_work_OT}

We first show in section \ref{ssec:patches}
that using image patches as features in the single image setting gives an elegant interpretation of the algorithm of~\cite{Gutierrez_ssvm2017} which combines the patch-based optimization framework from~\cite{Kwatra} with optimal transport.
In section~\ref{ssec:previous_work_OT}, we provide a review  of  models from the literature based on approximated optimal transport. We also detail how the proposed framework built upon semi-discrete optimal transport both encompasses these models and alleviates some of their caveats. 

\subsubsection{The specific case of patch representation}\label{ssec:patches}

We detail here the single image optimization framework in the case where we use the patch representation (as defined in section \ref{ssec:pyramid}). 
Recall that, as summarized in Alg.~\ref{alg:MStexgenGD}, all scales $l = 1, \ldots, N_F$ are treated simultaneously with GOTEX, and we aim at solving as many optimal transport problems. 
For the sake of simplicity, we focus on the single scale $l=1$ case in this paragraph. 
In this particular case with patch features, the functional $\FF$ from \eqref{eq:FF} writes
\begin{equation}
\FF(\theta,\psi) = \frac{1}{n} \sum_{i=1}^n \psi^c(P_i \theta) +   \frac{1}{m}\sum_{j=1}^m\psi_j,
\end{equation}
with
\begin{equation}
\psi^c(P_i\theta) = \min_j\left[c(P_i\theta, P_j\theta) -\psi_j\right].
\end{equation}

Now, we will see that using the quadratic cost $c$, and using a simple gradient descent for the update on $\theta$, the GOTEX algorithm \ref{alg:MStexgenGD} admits an interpretation in terms of iterated weighted nearest neighbor assignments.
Indeed, in the case $c(x,y) = \frac{1}{2}\|x-y\|^2$, Theorem~\ref{thm:2} ensures that
\begin{equation}\label{eq:grad_W2}
\nabla_\theta \FF(\theta^{k-1},\psi^{k}) =  \frac{1}{n} \left( \sum_{i=1}^nP_i^TP_i\theta^{k-1} - \sum_{i=1}^nP_i^TP_{\sigma^{k}(i)}\vv\right),
\end{equation}
at any point $(\theta^{k-1},\psi^{k})$ where we can uniquely define
\begin{equation}\label{eq:jstar}
\sigma^{k}(i) = \argmin_j \frac{1}{2}\| P_i\theta^k-P_j\vv \|^2 - \psi^{k}_j \quad \forall i = 1, \ldots, n .
\end{equation}
Notice that $P_j$ is a linear operator whose adjoint operator  $P_j^T$ maps a given patch $q$ to an image whose $j$-patch is $q$ and is zero elsewhere. Therefore $\sum_{i=1}^nP_i^T$ corresponds to a uniform patch aggregation. To simplify,  we consider periodic conditions for patch extraction, so that $\sum_{i=1}^nP_i^TP_i = s^2\Id$, where $s^2$ denotes the number of pixels in the $s \times s$ patches.
Hence, from \eqref{eq:grad_W2} and considering a step size $\eta \frac{n}{s^2}$, $\eta>0$, the update of $u$ through gradient descent writes
\begin{equation}
\theta^{k} = (1-\eta) \theta^{k-1} + \eta v^{k},
\end{equation}
where $v^{k} =  \frac{1}{s^2}\sum_{i=1}^n P_i^T P_{\sigma^{k}(i)} \vv$ is the image formed with patches from the exemplar image $\vv$  which are the nearest neighbors to the patches of $\theta^k$ in the sense of \eqref{eq:jstar}. 
The gradient step then mixes the current image $\theta^k$ with $v^k$. 
In the case $\psi = 0$, the minimum in \eqref{eq:jstar} is reached by associating to each patch of $\theta^k$ its $\ell_2$ nearest neighbor in the set $\{P_1\vv,\ldots, P_n\vv\}$. As described below, the case $\psi=0$ exactly corresponds to the texture optimization method~\cite{Kwatra}.

\subsubsection{Comparison with previous works on optimal patch transport}
\label{ssec:previous_work_OT}
\paragraph{Texture Optimization (TexOptim)}
In~\cite{Kwatra}, a multi-scale algorithm similar to ours is used for synthesis. It includes additional optimization ``tricks'' to accelerate the practical convergence.
First, the optimal transport cost between patch distributions is replaced by nearest-neighbor matching, which itself results in visible discrepancy between color distributions for texture synthesis (see e.g. Figure \ref{fig:patch_comp}, column b).
Additionally, a coarse-to-fine (a.k.a \emph{multi-grid}) approach is used, whereas every scales are optimized simultaneously in the proposed algorithm. At a given scale $l$ (using the Gaussian pyramid of patches defined in \ref{ssec:pyramid}), the problem originally formulated in~\cite{Kwatra} reads as an explicit patch matching
\begin{equation}\label{eq:textoptim_loss}
\sum_{i=1}^{n_l} 
    \min_{j} \|P_i^l \theta - P_j^l u_0\|^r
\end{equation}
where $r=0.8$ to enforce consistency between overlapping patches, in such a way that synthesized textures are local verbatim copies of the original image.
In practice, this non-convex optimization problem is addressed using an iterative reweighted least square methods (IRLS): denoting $\tilde \sigma^{k}(i)$ the nearest-neighbor (NN) matching defined from \eqref{eq:jstar} by setting $\psi^k_j = 0$
\begin{equation*}\label{eq:NN_assignment}
\tilde \sigma^{k}(i) = \argmin_j \frac{1}{2}\| P_i^l \theta^k - P_j^l\vv \|^2  \quad \forall \ i = 1, \ldots, n_l \; ,
\end{equation*}
the updated image $\theta^k$ at scale $l$ reads (again, considering periodic conditions)
$$
\theta^k(i) = \tfrac1{s^2} \sum_{i=1}^{n_l} w_i {P_i^l}^T P_{\tilde \sigma^k(i)}^l u_0
$$
where weights $w_i$ are defined according to IRLS 
$$
    w_i = \|P_{i}^l \theta^k - P_{\tilde \sigma^k(i)}^l u_0\|^{r-2}.
$$
Without statistical consistency guaranteed by optimal transport, an important aspect of this approach is the initialization: to ensure the statistical consistency at the coarsest scale, a random permutation of the patch of $u_0$ is used to set $\theta^0$.
Last, to avoid the excessive blurring resulting from overlapping patches, a stride of $s/4$ is used to sample large patches (ranging from $s=8$ to $s=32$ during optimization to enforce local copy).

\paragraph{Optimal Patch Assignment (OPA)}

To overcome the statistical inconsistency from nearest-neighbor matching,~\cite{Gutierrez_ssvm2017} 
enforces instead the optimal patch assignment between the discrete distribution of $n_l$ patches.
More precisely, the objective loss function now reads as (at a given scale $l$)
\begin{equation}\label{eq:OPA_loss}
 \min_{\sigma \in \Sigma_{n_l}} \sum_{i=1}^{n_l} 
    {\|P_i^l \theta - P_{\sigma(i)}^l u_0\|}_{1,2}
\end{equation}
where $\Sigma_{n}$ is the set of $n$ permutations, and
$\|.\|_{1,2}$ stands for the sum, over pixel coordinates $i$, of the Euclidean norm of pixelwise color vectors. 
By setting $r=1$, the problem of optimizing $\theta$ for a fixed assignment $\sigma$ corresponds to computing a color median. This helps reducing the blur obtained from averaging overlapping patches when using $r=2$ instead.
Rather than resorting to the proposed semi-discrete formulation, an Hungarian algorithm~\cite{Kuhn55thehungarian} is used instead to solve the optimal assignment between patches at each iteration.
Note that a similar but faster approach has been proposed in~\cite{webster2018innovative}, mainly approximating optimal assignments by soft assignments computed from the Sinkhorn Algorithm~\cite{cuturi2013sinkhorn}. 
The main drawbacks of these approaches is computation time and memory requirements to compute the assignment maps $\sigma$.
For instance, both methods require to compute and store the cost matrix for finding patch correspondences. This becomes prohibitive for large image size and it cannot benefit from GPU acceleration.

To accelerate convergence, as in~\cite{Kwatra}, large patches of $s=8$ pixels are therefore sampled on a decimated grid (i.e. stride of 2).
A Douglas-Rachford optimization algorithm is also required after each optimal patch assignment to compute the corresponding color median to update $\theta$. 
In Figure \ref{fig:patch_comp}, one can observe that the results are close to the ones obtained with the GOTEX Algorithm \ref{alg:MStexgenGD}, yet only using $4\times 4$ patches without stride nor median.

\paragraph{Sliced Wasserstein transport}

In~\cite{rabin2011wasserstein}, an approximate optimal transport cost function coined {\em Sliced Wasserstein distance} was introduced for texture synthesis and mixing. This framework is inspired by~\cite{portilla_simoncelli_2000} where an image is progressively synthesized from coarse to fine scale by sequentially matching its statistics to an exemplar image.
Originally, first and second order statistics of wavelet coefficients across scale and space 
were considered, restricting the generated textures to short range correlations.
In~\cite{rabin2011wasserstein}, the distribution of patches of wavelet coefficients was considered instead, to successfully synthesize large range structures. The projection of  patch distributions from the synthesized image to the desired distribution is ensured by minimizing the Sliced Wasserstein cost function using stochastic gradient descent.
Since then, it has been used in various imaging problems involving statistical comparison, such as in~\cite{kolouri2018sliced} to train auto-encoders.

To appropriately compare this approach to the proposed multi-scale patch-based optimization, we now consider adaptation of the GOTEX Algorithm \ref{alg:MStexgenGD} to the case of Sliced Wasserstein, in such a way that the loss function to minimize becomes
\begin{equation}\label{eq:MSloss_SW}
\mathcal{L}_{\text{SW}}(\theta) = \sum_{l=1}^{\LL} 
    \text{SW}^2(\mu^l_{\theta},\nu^l)
    .
\end{equation}

At a given scale $l$, the quadratic Sliced Wasserstein cost $\text{SW}^2$ between discrete distributions $\mu_\theta^l= \frac{1}{n}\sum_{i=1}^n \delta_{P_i \theta}$ and $\nu^l$ writes as the expected transport cost of the projected distributions
\begin{equation}\label{eq:SW}
    \text{SW}^2(\mu_\theta^l,\nu^l) =  
        \E_{\omega \sim \mathcal{U}(\mathbb{S}^{d})}
        \min_{\sigma \in \Sigma_{n_l}}
        \frac1{n_l} \sum_{i=1}^{n_l} 
        \langle P_i^l \theta - P_{\sigma(i)}^l u_0, \omega \rangle^2
\end{equation}
where $\mathcal{U}(\mathbb{S}^d)$ indicates the uniform distribution on the $d$-dimensional sphere.
The gradient $\mathcal{L}_{\text{SW}}$ can be written explicitly~\cite{rabin2011wasserstein}, but one can also rely on auto-differentiation as done for instance in~\cite{kolouri2018sliced} and for our experiments.
The main limitations of this approximation are, to begin with,
that it requires the two distributions to have the same number of samples $n$ with uniform probability, 
that the required number $k$ of drawn directions $\omega$ for stochastic optimization increases with the dimension $d$ of the features,
and that it requires to sort values which cannot benefit from GPU acceleration.

\subsection{Experimental comparisons}\label{sec:exp_image_optim}
We  now illustrate the benefit of the GOTEX method with respect to related methods in the literature. We first focus on patch-based method and then study VGG features.

\subsubsection{Comparisons with patch-based methods}\label{ssec:patch_comp}

Figure \ref{fig:patch_comp} illustrates the results of different methods discussed in section \ref{ssec:previous_work_OT} using image optimization based solely on patch representation. 
As already mentioned, when comparing TexOptim (Fig.~\ref{fig:patch_comp}.b) to OPA (Fig.~\ref{fig:patch_comp}.c), or any other method based on optimal transport (Fig.~\ref{fig:patch_comp}.d and e), the lack of statistical consistency is visually striking.
While benefiting from a noticeable speed-up when approximating OPA with Sliced Wasserstein, the resulting texture is more blurry (mostly noticeable within the flowers in Fig.~\ref{fig:patch_comp}.d). 
This is mainly due to the fact that the optimization of SW is stochastic by nature, sampling randomly new directions at each step.
By contrast, the proposed algorithm manages to produce qualitative results (Fig.~\ref{fig:patch_comp}.e), without any approximations.

\begin{figure}[ht!]
\setlength{\flen}{0.19\linewidth}
\centering
\setlength{\tabcolsep}{2pt}
\begin{tabular}{c cccc}
    \includegraphics[width=\flen]{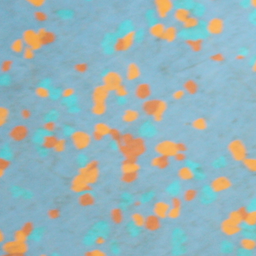}
    &\includegraphics[width=\flen]{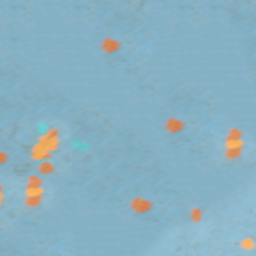}
    &\includegraphics[width=\flen]{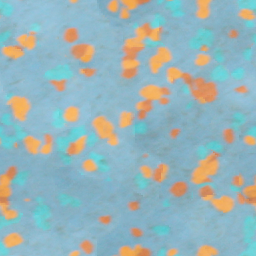}
    &\includegraphics[width=\flen]{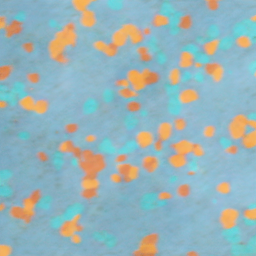}
    &\includegraphics[width=\flen]{generated_ground1013_small_256_to_512.png}
    \\
    \includegraphics[width=\flen]{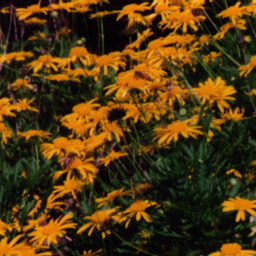}
    &\includegraphics[width=\flen]{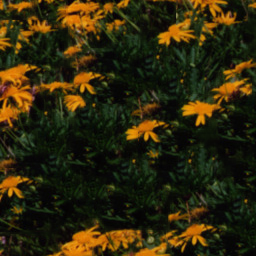}
    &\includegraphics[width=\flen]{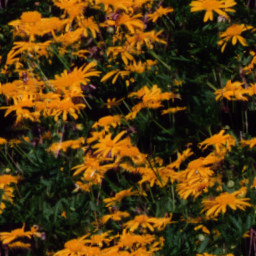}
    &\includegraphics[width=\flen]{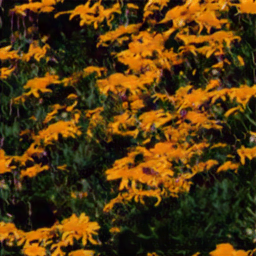}
    &\includegraphics[width=\flen]{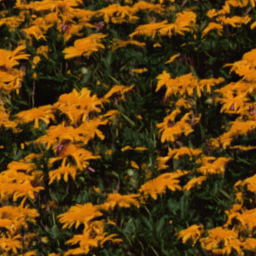}
    \\
    \footnotesize 
      a. Sample %
    & b. TexOptim~\cite{Kwatra}
    & c. OPA~\cite{Gutierrez_ssvm2017}
    & d. SW~\cite{rabin2011wasserstein}
    & e. GOTEX-patch
    \\
\end{tabular}
\caption{Comparison of exemplar-based texture synthesis methods using patch-based image optimization and various optimal transport approximations.
a) displays the exemplar image.
b) shows the texture optimization results from ~\cite{Kwatra} using coarse-to-fine optimization and nearest-neighbor patch assignment.
c) uses instead Optimal Patch Assignment (OPA)~\cite{Gutierrez_ssvm2017}, 
enforcing the target patch distribution. 
d) is based on the proposed loss function \eqref{eq:OTpatchloss} where OT is approximated by Sliced Wasserstein (SW) cost.
e) is the proposed semi-discrete approach, solving more accurately the OT problem.
See the text for more details.
}
\label{fig:patch_comp}
\end{figure}

\subsubsection{Comparison with VGG-based optimization}
\label{ssec:comp_VGG}

Figure \ref{fig:comp_VGG} compares the synthesized results from methods based on VGG features exclusively. 
As reported in~\cite{heitz2021sliced} and illustrated here,
using Sliced-Wasserstein gives slightly better results than the original approach based on Gram matrices regarding color consistency, even when using the post-processing based on color histogram matching advocated in~\cite{gatys_texture_2015}.
Our GOTEX-VGG approach produces decent results, with the same color inconsistencies than the one observed with the method of Gatys et al.~\cite{gatys_texture_2015}. As discussed in section~\ref{ssec:comp_loss}, this issue can be solved with the GOTEX-mix loss \eqref{eq:OTmixedloss} that constrains both VGG and patch distributions.

\begin{figure}[!ht]
\setlength{\flen}{0.24\linewidth}
\centering
\setlength{\tabcolsep}{2pt}
\begin{tabular}{cccc}
\includegraphics[width=\flen]{ground1013_small_256.png}
&\includegraphics[width=\flen]{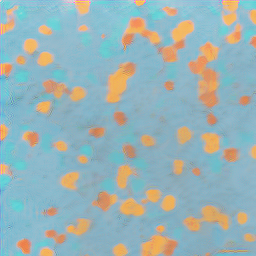}
&\includegraphics[width=\flen]{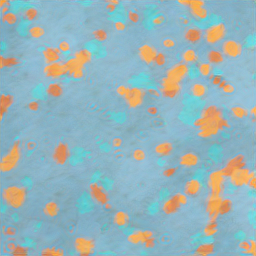}
&\includegraphics[width=\flen]{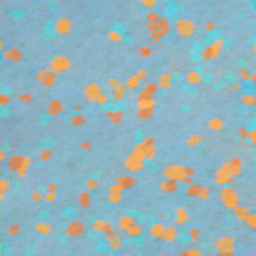}
\\
\includegraphics[width=\flen]{raad3_256.png}
&\includegraphics[width=\flen]{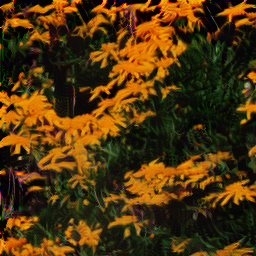}
&\includegraphics[width=\flen]{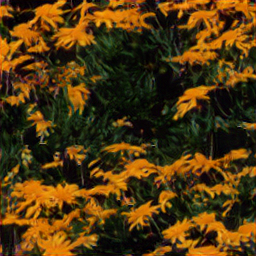}
&\includegraphics[width=\flen]{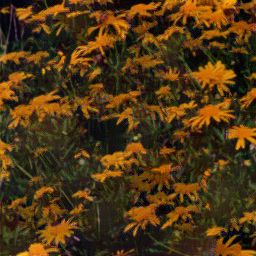}%
\\
  \footnotesize Sample %
& \footnotesize Gram-VGG~\cite{gatys_texture_2015}
& \footnotesize SW-VGG~\cite{heitz2021sliced}
& \footnotesize GOTEX-VGG 
\\
\end{tabular}
\caption{As in Figure \ref{fig:patch_comp}, we compare the results of the proposed Algorithm \ref{alg:MStexgenGD} (GOTEX-VGG) for image optimization based on VGG (already exposed in Fig.~\ref{fig:compfeature}), as pioneered by~\cite{gatys_texture_2015} (Gram-VGG) and with the Sliced Wasserstein approximation (SW-VGG) studied in~\cite{heitz2021sliced}.\label{fig:comp_VGG}}
\end{figure}

\subsection{Other applications}\label{ssec:other_app}

We finally present adaptations of the GOTEX framework to tackle two related synthesis problems: texture inpainting and texture interpolation.

\subsubsection{Texture inpainting}\label{sec:inpainting}

The framework using the OT-patch loss can be extended to texture inpainting, by taking the patches outside a masked area as the target ones. By optimizing only the pixels within the masked area, the very same algorithm yields an efficient texture inpainting method, as illustrated in Figure~\ref{fig:inpainting}.
Since the comparison of patch distributions with OT allows to capture highly non-Gaussian behavior, this texture inpainting scheme can treat highly structured textures, as opposed to the Gaussian model of~\cite{galerne2017texture} which can only deal with unstructured textures.
Besides, the fact that the optimization problem can be naturally restricted to the masked area permits, again, to inherently solve patch aggregation problems.
In contrast to the algorithm of~\cite{leclaire2021jmiv} based on an ad-hoc aggregation technique at the border of the mask, our inpainting results do not suffer from the same blur artifacts on this region, as can be seen on Fig.~\ref{fig:inpainting}.

\begin{figure}[!ht]
\setlength{\flen}{0.12\linewidth}
\vskip 0.in
\begin{center}
\includegraphics[height=\flen]{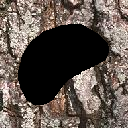}
\includegraphics[height=\flen]{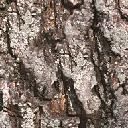}
\includegraphics[height=\flen]{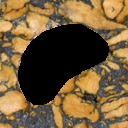} 
\includegraphics[height=\flen]{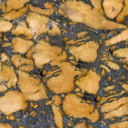} 
\includegraphics[height=\flen]{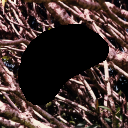}
\includegraphics[height=\flen]{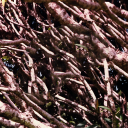}
\includegraphics[height=\flen]{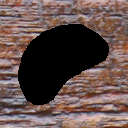}
\includegraphics[height=\flen]{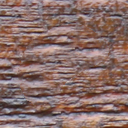} 

\vskip1pt

\includegraphics[height=\flen]{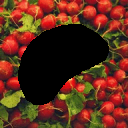}
\includegraphics[height=\flen]{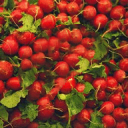}
\includegraphics[height=\flen]{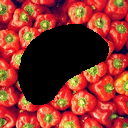} 
\includegraphics[height=\flen]{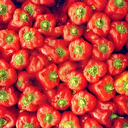} 
\includegraphics[height=\flen]{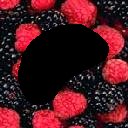}
\includegraphics[height=\flen]{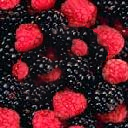} 
\includegraphics[height=\flen]{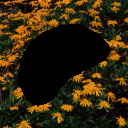}
\includegraphics[height=\flen]{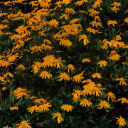}

\vskip1pt

\includegraphics[height=\flen]{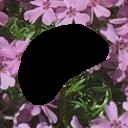}
\includegraphics[height=\flen]{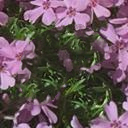}
\includegraphics[height=\flen]{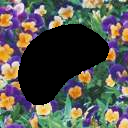}
\includegraphics[height=\flen]{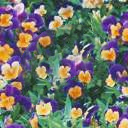}
\includegraphics[height=\flen]{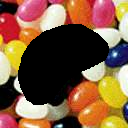} 
\includegraphics[height=\flen]{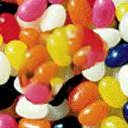} 
\includegraphics[height=\flen]{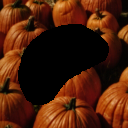}
\includegraphics[height=\flen]{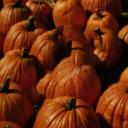}

\vskip1pt 

\includegraphics[height=\flen]{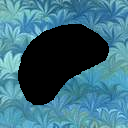}
\includegraphics[height=\flen]{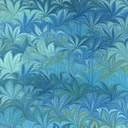}
\includegraphics[height=\flen]{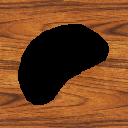}
\includegraphics[height=\flen]{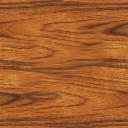}
\includegraphics[height=\flen]{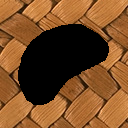}
\includegraphics[height=\flen]{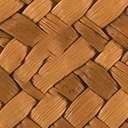}
\includegraphics[height=\flen]{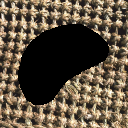}
\includegraphics[height=\flen]{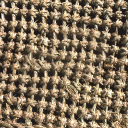}

\vskip1pt 

\begin{minipage}[t]{\flen}
\mysmall{Masked}
\end{minipage} 
\begin{minipage}[t]{\flen}
\mysmall{Inpainted}
\end{minipage}
\begin{minipage}[t]{\flen}
\mysmall{Masked}
\end{minipage} 
\begin{minipage}[t]{\flen}
\mysmall{Inpainted}
\end{minipage}
\begin{minipage}[t]{\flen}
\mysmall{Masked}
\end{minipage} 
\begin{minipage}[t]{\flen}
\mysmall{Inpainted}
\end{minipage}
\begin{minipage}[t]{\flen}
\mysmall{Masked}
\end{minipage} 
\begin{minipage}[t]{\flen}
\mysmall{Inpainted}
\end{minipage}
\caption{Texture inpainting on various masked textures of size $128\times 128$ with $s=4$ and $L=3$ using a slightly adapted version of Alg.~\ref{alg:MStexgenGD}. The method is able to fill properly the masked region while agreeing quite convincingly with the surrounding content around the mask boundary.}
\label{fig:inpainting}
\end{center}
\vskip -0.1in
\end{figure}

\subsubsection{Texture barycenters}\label{sec:tex_mixing}

We now tackle the problem of interpolating between different exemplar textures. To that end, we use the GOTEX barycentric loss \eqref{eq:bar_loss} presented in section \ref{sec:bary}. For $K=2$, it corresponds to finding a new texture $\theta$ whose feature distribution is a Wasserstein interpolation (in the sense of~\eqref{eq:bar_loss}) of the feature distributions extracted from the two exemplar textures $u_0$ and $u_1$. We provide in Fig.~\ref{fig:barycenters2} the interpolation results $\theta_t$ between various textures $u_0$ and $u_1$, for both patchs and VGG-19 features. In these experiments, the barycenter $\theta_t$ corresponds to the solution of \eqref{eq:bar_loss} for different weights  $\alpha_0=1-\alpha_1=t\in\{0.2,\cdots,0.8\}$. We also compare these results to the framework of~\cite{gatys_texture_2015} that realizes  texture interpolation by minimizing the weighted sum of Gram losses \eqref{eq:gram_loss} of VGG features. Since the distance between Gram matrices of features does not represent a distance between feature distributions, this method produces textures with distinct spatial parts that either belong to one or to the other texture $u_0$ and $u_1$. On the other hand, our approach  properly interpolates between texture contents.

\newlength{\bary}
\setlength{\bary}{0.09\linewidth}
\newcommand{\sidecapt}[1]{ {\begin{sideways}\parbox{1.05\bary}{\small #1}\end{sideways}}}

\newcommand{\expeinterp}{
\begin{minipage}{0.14\linewidth}\vspace{-0.25cm}
\centering\includegraphics[width=\bary]{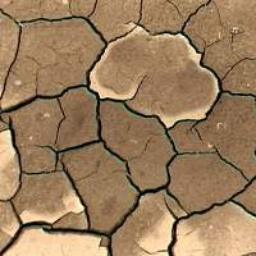} %
\end{minipage}
\begin{minipage}{0.55\linewidth}
\begin{minipage}[b]{.5\bary}
\centering
\sidecapt{\centering GOTEX\\Patches}
\end{minipage}
\foreach \x in {0.2,0.4,0.5,0.6,0.8}%
{\includegraphics[width=\bary]{patch_barycenter_\nameA_resized256_\x_\nameB_resized256_final.png} }\\%
\begin{minipage}[b]{.5\bary}
\centering
\sidecapt{\centering GOTEX\\VGG}
\end{minipage}
\foreach \x in {0.2,0.4,0.5,0.6,0.8}
{\includegraphics[width=\bary]{barycenter_\nameA_resized256_\x_\nameB_resized256_final.png} }\\
\begin{minipage}[b]{.5\bary}
\centering
\sidecapt{\centering Gram\\VGG}
\end{minipage}
\foreach \x in {0.2,0.4,0.5,0.6,0.8}
{\includegraphics[width=\bary]{gram_barycenter_\nameA_resized256_\x_\nameB_resized256_final.png} }\\
\end{minipage}
\begin{minipage}{0.14\linewidth}\vspace{-0.25cm}
\centering\includegraphics[width=\bary]{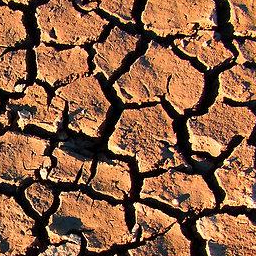}
\end{minipage}\vspace{-0.25cm}
}

\begin{figure}[p]
\centering
\def\nameA{cracked_0162}\def\nameB{cracks}
\expeinterp

\def\nameA{veined_0142}\def\nameB{veined_0066}
\expeinterp

\def\nameA{Shair4}\def\nameB{Shair3}
\expeinterp

\def\nameA{striped_0069}\def\nameB{striped_0085}
\expeinterp

\begin{minipage}{0.14\linewidth}
\centering$u_0$
\end{minipage} 
\begin{minipage}{0.55\linewidth}
{\begin{minipage}{.5\bary}
\centering
\end{minipage} }
\foreach \x in {0.2,0.4,0.5,0.6,0.8}
{\begin{minipage}{\bary}
\hfill
$\theta_{\x}$
\end{minipage} }
\end{minipage}
\begin{minipage}{0.14\linewidth}
\centering$u_1$
\end{minipage} 

\caption{Interpolations $\theta_t$ between exemplar textures $u_0$ and $u_1$, obtained by minimizing the barycentric GOTEX loss \eqref{eq:bar_loss} with $\alpha_0=1-\alpha_1=t\in\{0.2,\cdots ,0.8\}$ for either patches or VGG features. Barycenters computed with respect to Gram losses \eqref{eq:gram_loss} of VGG features, as in~\cite{gatys_texture_2015}, are presented for comparison. The approach of~\cite{gatys_texture_2015} just copies and pastes different parts of the two exemplar textures, whereas the GOTEX framework realizes visually plausible texture interpolations for any kind of involved feature.}
\label{fig:barycenters2}

\end{figure}

%
\section{Generative model}\label{sec:generator}
%

In this section, we consider the problem of training a network to generate images that have prescribed feature distributions at multiple scales. To do so, we use our framework with the generative model $g_\theta$ being the feed-forward deep convolutional network introduced in~\cite{ulyanov2016texture} for texture generation. Then we present some visual results together with a comparison with existing methods. 

Estimating a generative model corresponds to a semi-discrete problem, where one to be able to sample a continuous distribution $g_\theta\sharp\zeta$ which feature distribution $\mu_\theta = \frac{1}{n}\sum_{i=1}^n \left(F_i\circ g_\theta\right)\sharp\zeta$ fits the one of a discrete example distribution $\nu=\sum_{j=1}^m\delta_{y_j}$.
The semi-discrete formulation of optimal transport is thus perfectly adapted to deal with the estimation of a  generative model.

\subsection{Experimental setting}

\paragraph{Neural network architecture}
We consider the feed-forward texture synthesis model introduced in~\cite{ulyanov2016texture} that we refer to as \emph{TexNet}. This model takes as input a latent variable $z = \{z_0,\ldots,z_M\}$ constituted of $M+1$ inputs $z_l$ of size $\propto 2^l$ (with $M=4$ as advocated in~\cite{ulyanov2016texture} for texture synthesis). 
The first input $z_0$ passes through a convolutional block and a block of upsampling by a factor $2$, then the result is concatenated with $z_1$ and passes again through a convolutional and upsamling block and so on. For the detailed architecture, we refer to the Figure~2 of~\cite{ulyanov2016texture}. 

One of the key advantage of this network in comparison to other methods is its small number of parameters (around 65K), enabling fast convergence during optimization on a single image.
It has been originally designed to synthesize textures by minimizing the Gram-VGG loss introduced in~\cite{gatys_texture_2015}. 
We next demonstrate that the parameters of such a generative network can be learned with the GOTEX algorithm, \emph{i.e.} by only enforcing the feature distributions at various scales.

\paragraph{Optimization}
In our PyTorch implementation, we use this time the Adam  optimizer~\cite{kingma2014adam} to estimate the parameters $\theta$. For the GOTEX-patch algorithm, $10000$ iterations have been used with a learning-rate $\eta_{\theta} = 0.01$. An averaged stochastic gradient ascent with $100$ inner iterations is used for computing $\psi^*$.
In total, one million $4\times4$ patches at $\LL = 4$ different scales are therefore sampled to train the neural network.
For GOTEX-VGG, $10000$ images are sampled during training (batch of 1) with the same optimizer.
In these setting, 10 hours are required to train each generator with a GPU Nvidia K40m.

\paragraph{Comparared methods}
For experiments, Pytorch implementations of SinGAN~\footnote{\url{github.com/tamarott/SinGAN}}, 
PSGAN~\footnote{\url{github.com/zalandoresearch/famos}} 
and TexNet~\footnote{\url{github.com/JorgeGtz/TextureNets_implementation}} were run with their default parameters.
Regarding TexTo~\cite{leclaire2021jmiv}, a similar experimental setting has been used, with $10^7$ iterations of ASGD and $7\times7$ patches at $4$ different scales. Note that a pre-processing step is required to perform the bi-level clustering of the target patch distribution.

The Sliced Wasserstein loss used to approximate optimal transport within our framework is optimized using $5000$ iterations of the Adam optimizer. At each step, all the $4 \times 4$ patches are extracted from a $256 \times 256$ synthesized texture (i.e. a batch of 1 image).
The loss itself is based on $d$ random directions in addition to the canonical basis, where $d=48$ is the patch dimension.

\subsection{Experimental results and discussion}

Figures~\ref{fig:CNNgen_big_comp} and~\ref{fig:CNNgen_big_comp2} display four synthesized textures with the proposed GOTEX framework and five relevant synthesis methods from the literature based on patches or deep features.

\paragraph{Patch representation}
We first compare GOTEX-patch (Fig.~\ref{fig:CNNgen_big_comp}.b) with the Sliced Wasserstein approximation (Fig.~\ref{fig:CNNgen_big_comp}.c) discussed in the previous section. 
Since the SW cost is computed at each iteration, the optimization problem now requires random sampling of directions for patch projections in addition to the image sampling itself. The same architecture has been used to train both network. 
While the results obtained with GOTEX tends to be slightly over-smoothed, it is much more noisier with SW. 
This is particularly noticeable in the results of columns 2 and 3 where the original texture is regular and the syntheses with SW contain high frequencies artifacts.

\begin{figure}[!ht]
\setlength{\flen}{0.09\linewidth}
\newlength{\halflen}
\setlength{\halflen}{0.027\linewidth}
\newcommand{\sidecap}[1]{ {\begin{sideways}\parbox{2\flen}{\centering #1}\end{sideways}} }
\newcommand{\sidecapO}[1]{ {\begin{sideways}\parbox{1.2\flen}{\centering #1}\end{sideways}} }
\centering
\setlength{\tabcolsep}{2pt}
\newcommand\Tstrut{\rule{0pt}{2.6ex}}         
\newcommand\Bstrut{\rule[-0.9ex]{0pt}{0pt}}   
\begin{center}
\begin{tabular}{c c cccc}
\sidecapO{a. Original}
&
&\raisebox{\halflen}{\includegraphics[height=\flen]{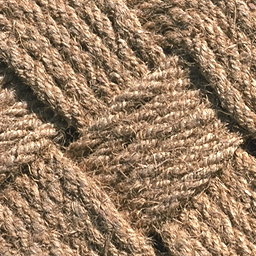}}~%
&\raisebox{\halflen}{\includegraphics[height=\flen]{ground1013_small_256.png}}%
&\raisebox{\halflen}{\includegraphics[height=\flen]{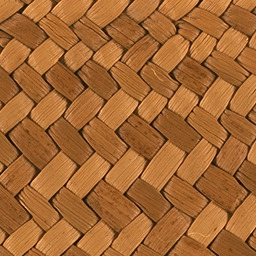}}%
&\raisebox{\halflen}{\includegraphics[height=\flen]{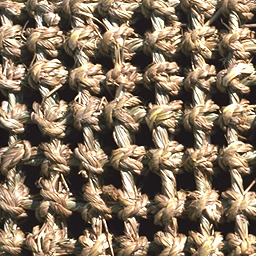}}
\\
\sidecap{b. GOTEX}
&\sidecap{(patch)}
&\includegraphics[height=2\flen]{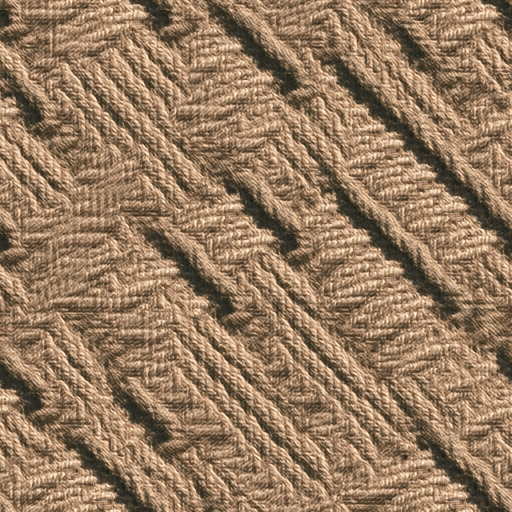}%
&\includegraphics[height=2\flen]{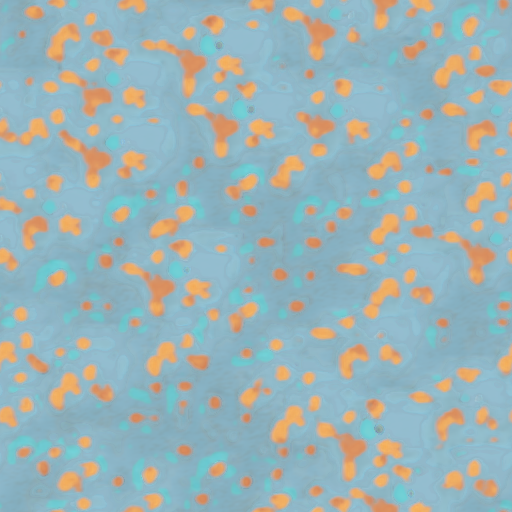}%
&\includegraphics[height=2\flen]{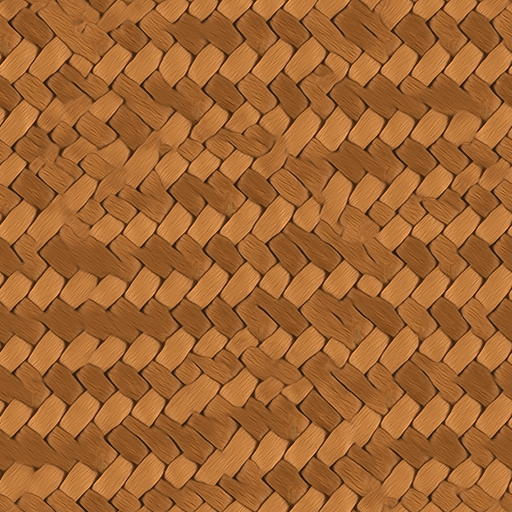}%
&\includegraphics[height=2\flen]{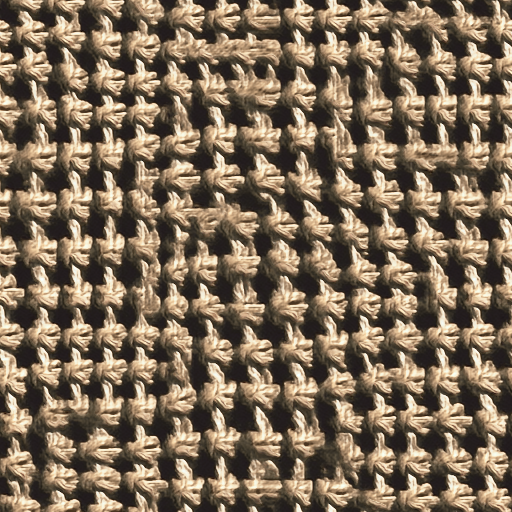}
\\
\sidecap{c. SW}
&\sidecap{(patch)}
&\includegraphics[height=2\flen]{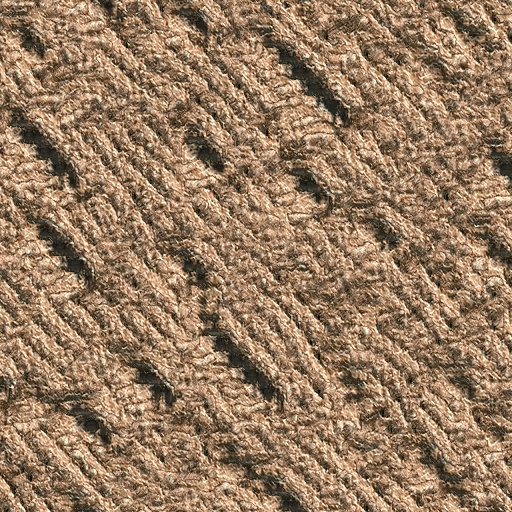} 
&\includegraphics[height=2\flen]{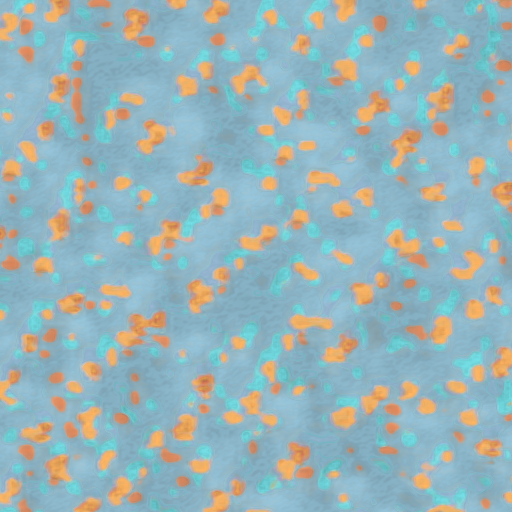} 
&\includegraphics[height=2\flen]{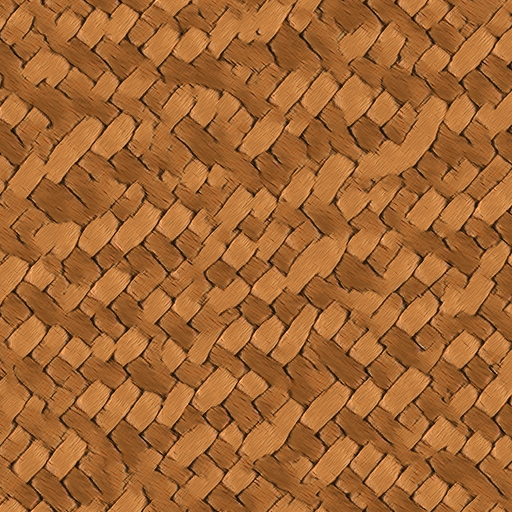} 
&\includegraphics[height=2\flen]{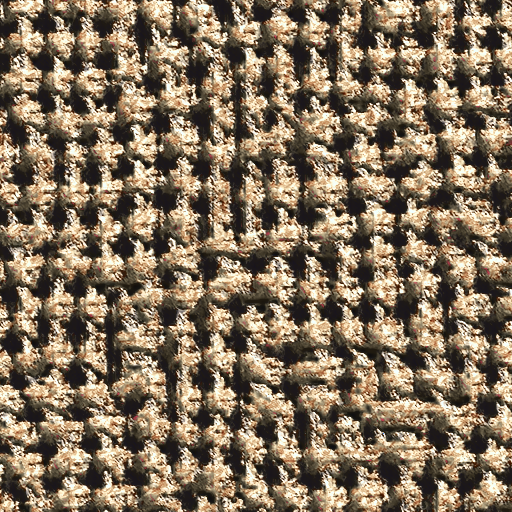}  
\\
\sidecap{d. TexTo}
&\sidecap{(patch)}
&\includegraphics[height=2\flen]{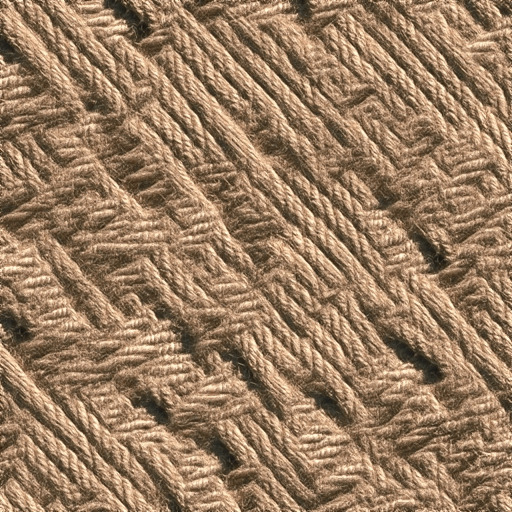} 
&\includegraphics[height=2\flen]{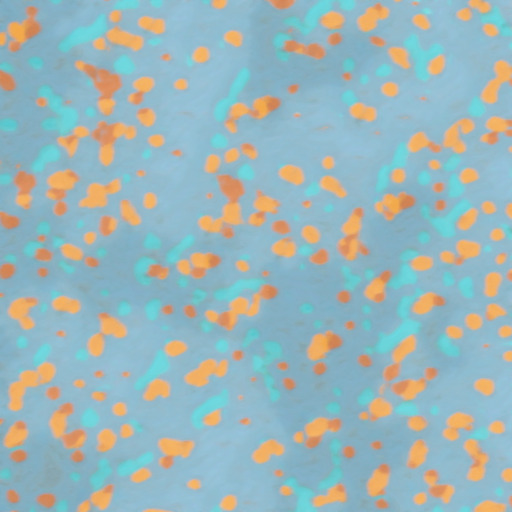} 
&\includegraphics[height=2\flen]{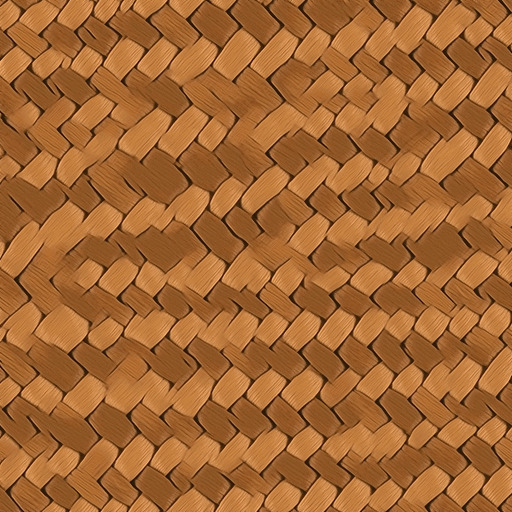} 
&\includegraphics[height=2\flen]{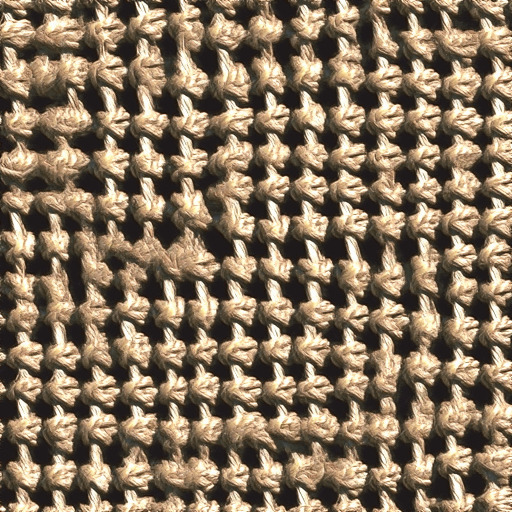}
\end{tabular}
\end{center}
\caption{Texture synthesis from a generative neural network trained on a single $256\times256$ sample (a). 
Our GOTEX-patch multi-scale approach (b) using $4\times4$ patches (see Alg.~\ref{alg:MStexgenGD}) is first compared with 
(c) multi-scale Sliced Wasserstein (SW) approximation, and 
(d) TexTo~\cite{leclaire2021jmiv}, using an explicit network of multi-layer transportation maps on $7\times7$ patches rather than a convolutional network.
}
\label{fig:CNNgen_big_comp}
\end{figure}

We also compare to TexTo, the multi-layer OT network (Fig.~\ref{fig:CNNgen_big_comp}.d) proposed in~\cite{leclaire2021jmiv} where the convolutional network is replaced by approximated multi-level optimal transport maps on $7\times 7$ patches that are averaged to synthesize an image.
A major difference with the proposed method is that GOTEX solves a global optimization problem that tackles the multiscale patch distributions in a direct manner.
On visual inspection, results are fairly similar, even if the proposed method does not rely on any coarse to fine optimization nor multi-level approximation. 
The patch aggregation step from TexTo may yield more blur than using a generative network which inherently deals with this issue.
This is supported by the in-depth comparison provided in~\cite{Houdard_ssvm21} based on various similarity metrics.

\begin{figure}[tp]
\setlength{\flen}{0.08\linewidth}
\setlength{\halflen}{0.025\linewidth}
\newcommand{\sidecap}[1]{ {\begin{sideways}\parbox{2\flen}{\centering #1}\end{sideways}} }
\newcommand{\sidecapO}[1]{ {\begin{sideways}\parbox{1.2\flen}{\centering #1}\end{sideways}} }
\centering
\setlength{\tabcolsep}{2pt}
\newcommand\Tstrut{\rule{0pt}{2.6ex}}         
\newcommand\Bstrut{\rule[-0.9ex]{0pt}{0pt}}   
\begin{center}
\begin{tabular}{c c cccc}
\sidecapO{a. Original}
&
&\raisebox{\halflen}{\includegraphics[height=\flen]{cordes_256.png}}~%
&\raisebox{\halflen}{\includegraphics[height=\flen]{ground1013_small_256.png}}%
&\raisebox{\halflen}{\includegraphics[height=\flen]{raad1_256.png}}%
&\raisebox{\halflen}{\includegraphics[height=\flen]{raad8_256.png}}
\\

    \sidecap{b. GOTEX}
    &\sidecap{(mix)}
    &\includegraphics[height=2\flen]{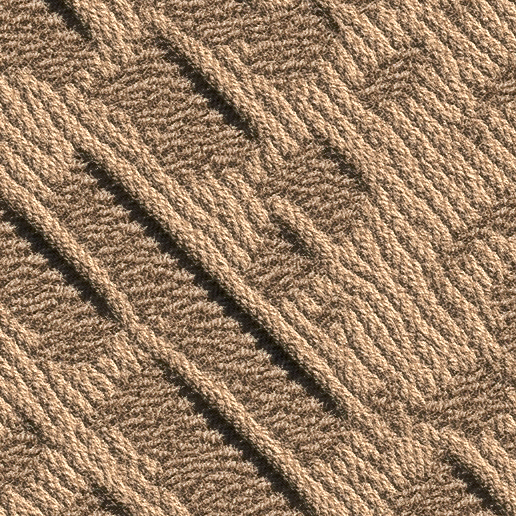}%
    &\includegraphics[height=2\flen]{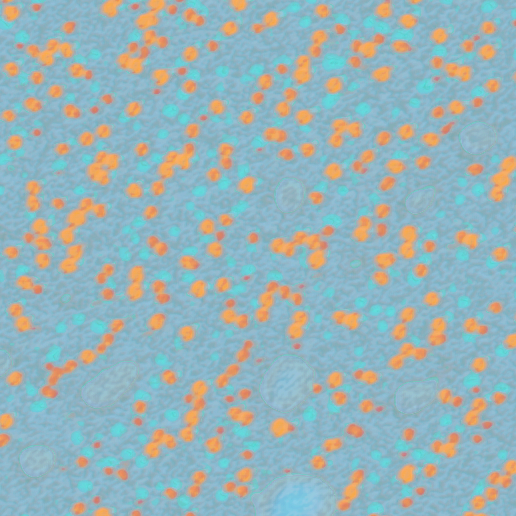}%
    &\includegraphics[height=2\flen]{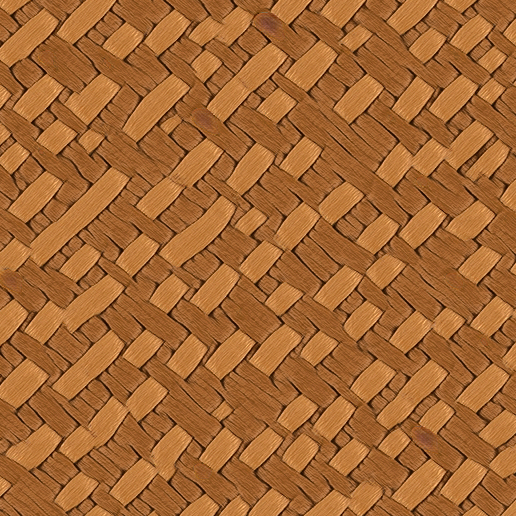}%
    &\includegraphics[height=2\flen]{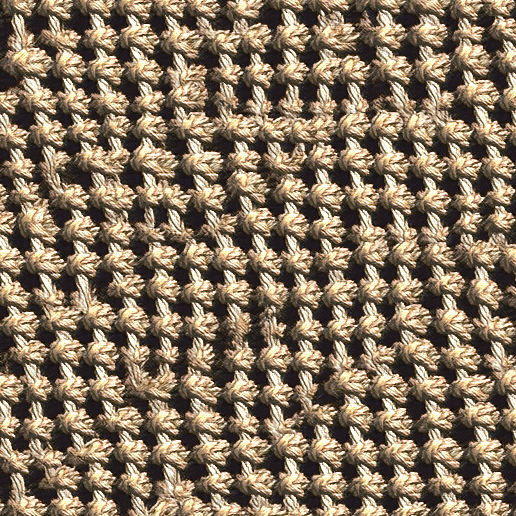}
    \\
    \sidecap{c. GOTEX}
    &\sidecap{(VGG)}
    &\includegraphics[height=2\flen]{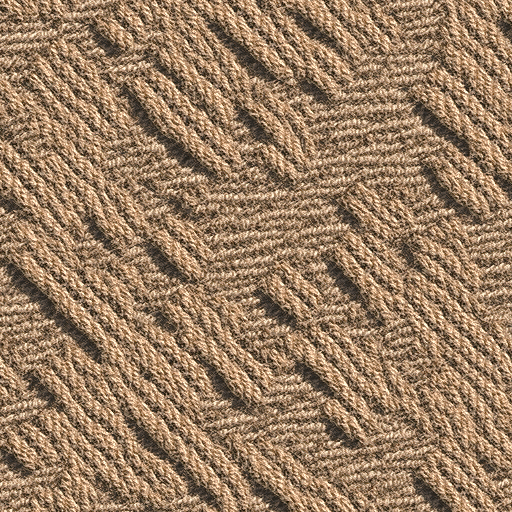}%
    &\includegraphics[height=2\flen]{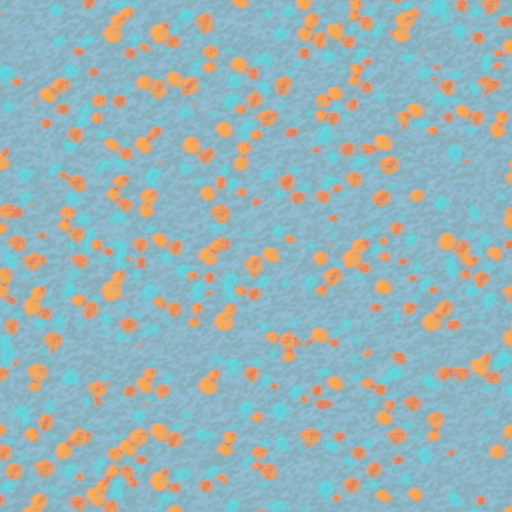}%
    &\includegraphics[height=2\flen]{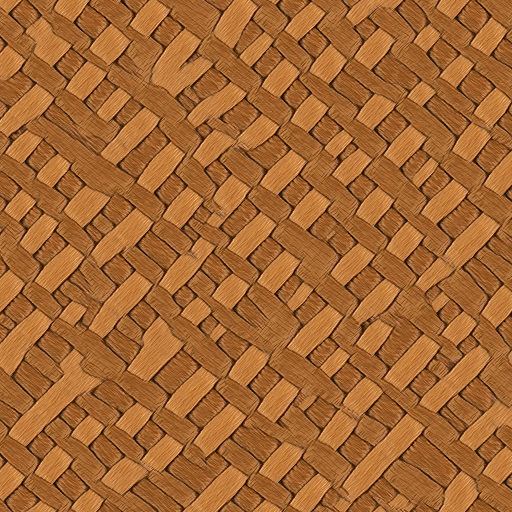}%
    &\includegraphics[height=2\flen]{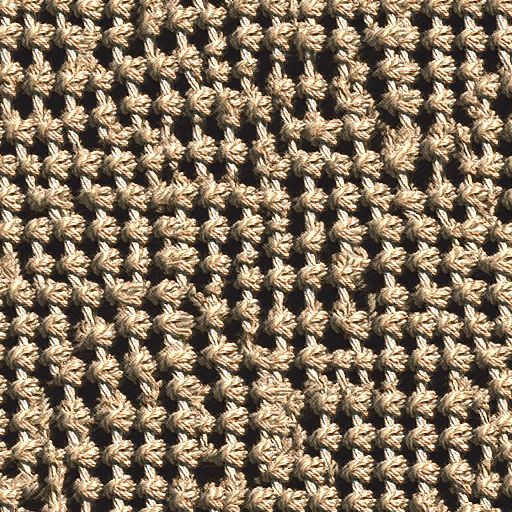}

\\
\sidecap{d. SW}
&\sidecap{(VGG)} 
& \includegraphics[height=2\flen]{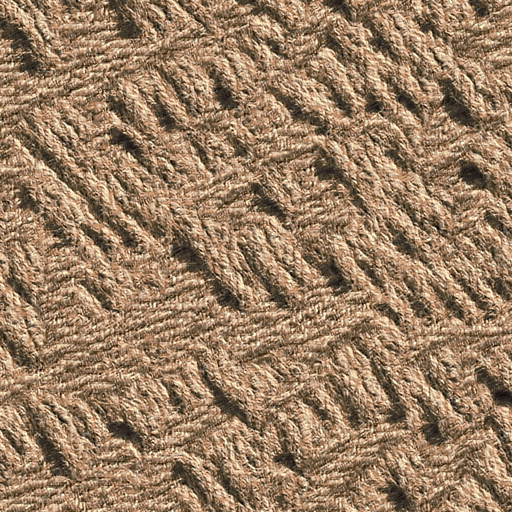}
& \includegraphics[height=2\flen]{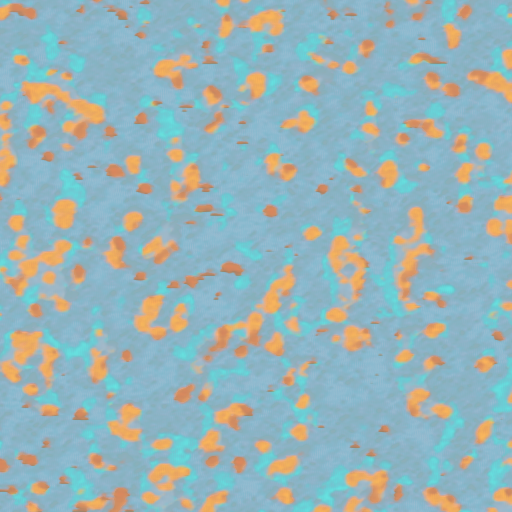}
& \includegraphics[height=2\flen]{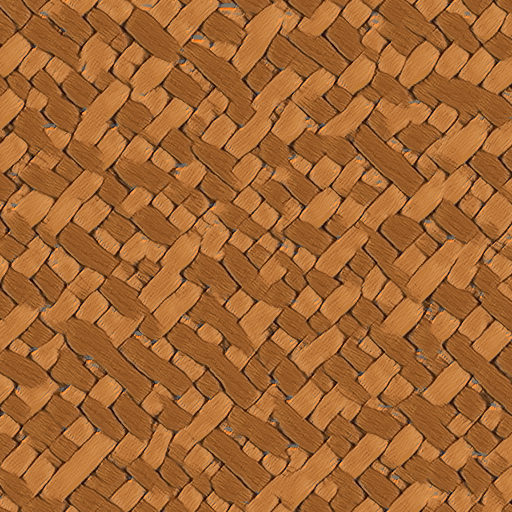}
& \includegraphics[height=2\flen]{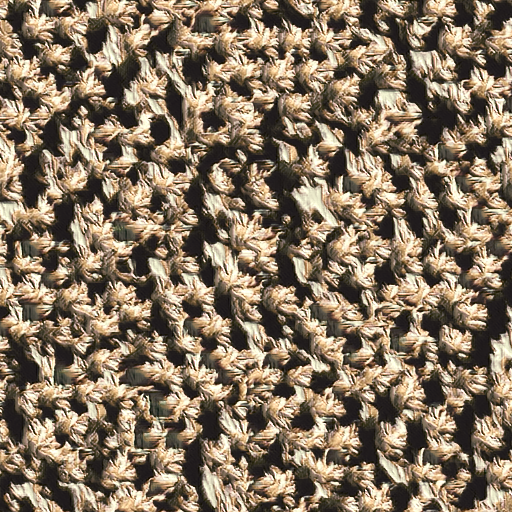}
\\
\sidecap{e. TexNet}
&\sidecap{(VGG)}
&\includegraphics[height=2\flen]{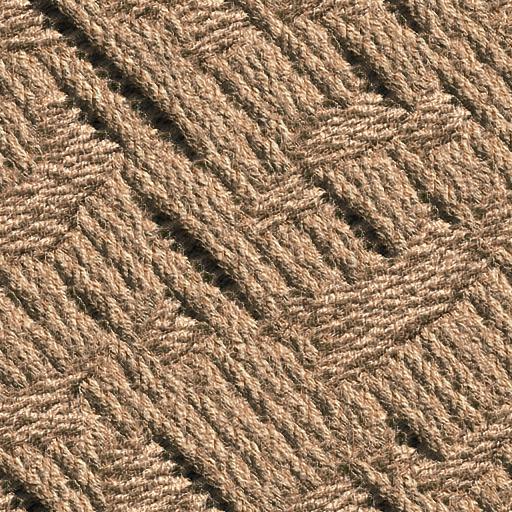}%
&\includegraphics[height=2\flen]{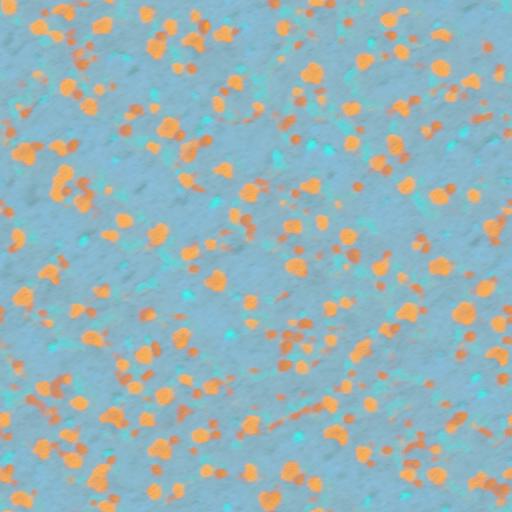}%
&\includegraphics[height=2\flen]{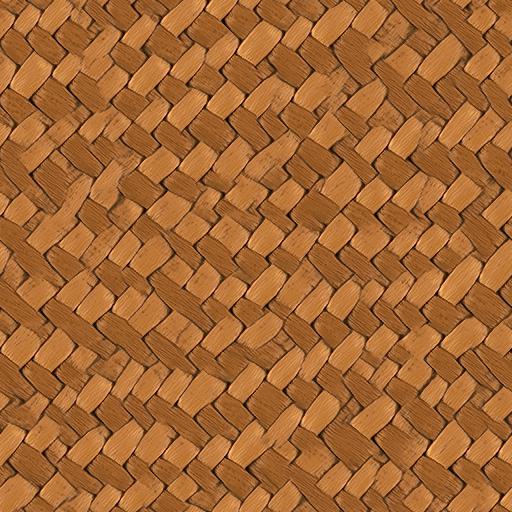}%
&\includegraphics[height=2\flen]{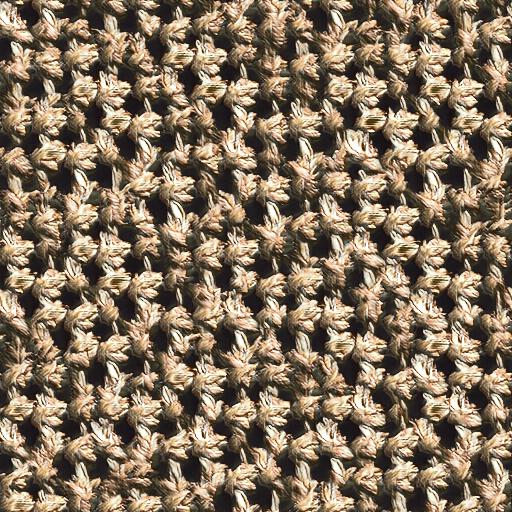}
\\
\sidecap{f. SinGAN}
&\sidecap{(patch + Adv.)}
&\includegraphics[height=2\flen]{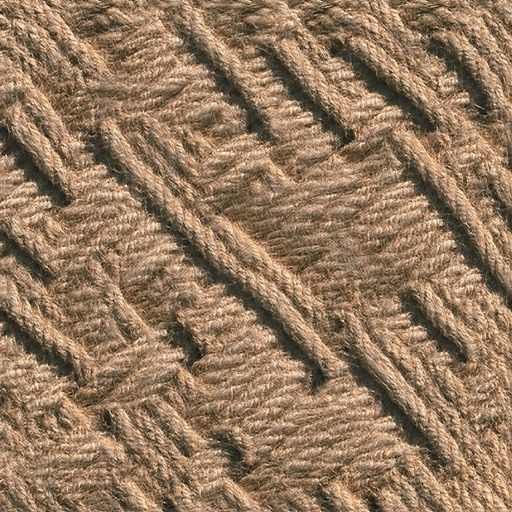}%
&\includegraphics[height=2\flen]{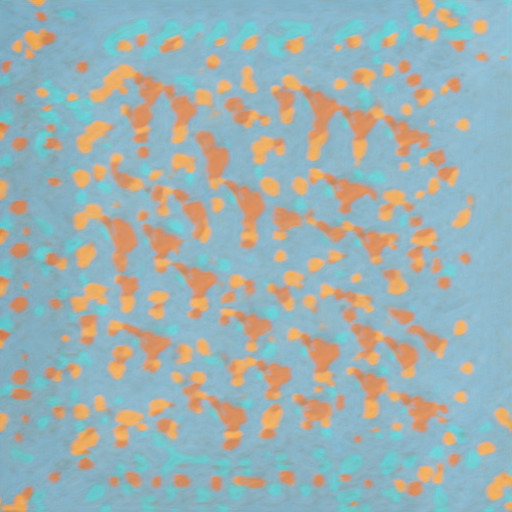}%
&\includegraphics[height=2\flen]{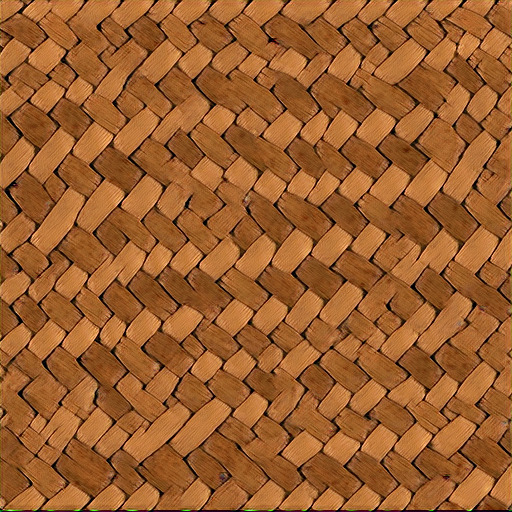}%
&\includegraphics[height=2\flen]{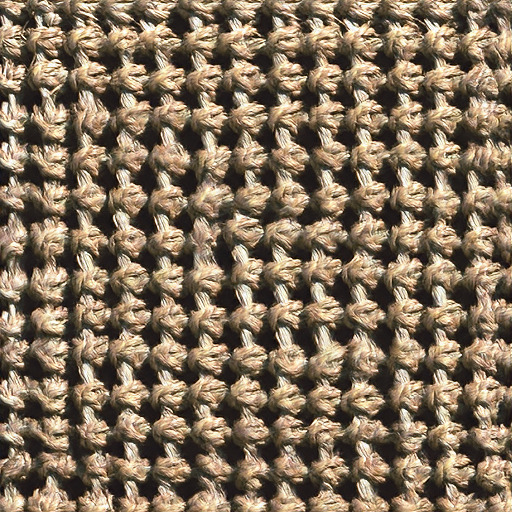}\
\\
\sidecap{g. PSGAN} 
&\sidecap{(patch + Adv.)}
&\includegraphics[height=2\flen]{PSGAN_cordes_crop}%
&\includegraphics[height=2\flen]{PSGAN_ground1013_crop}%
&\includegraphics[height=2\flen]{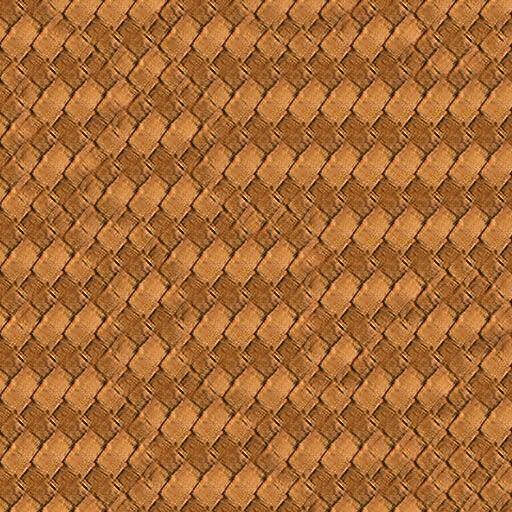}%
&\includegraphics[height=2\flen]{PSGAN_raad8_crop}
\\

\end{tabular}
\end{center}

\caption{Texture synthesis from a generative neural network trained on a single $256\times256$ sample (a). 
Methods 
(b) GOTEX-mix,
(c) GOTEX-VGG, 
(d) SW-VGG and
(e) TexNet~\cite{ulyanov2016texture} use pretrained VGG features.
Deep neural features are used for 
(f) SinGAN~\cite{shaham2019singan} and 
(g) PSGAN~\cite{bergmann2017learning}, where an adversarial network is trained on patches.
}
\label{fig:CNNgen_big_comp2}
\end{figure}

\paragraph{Perceptual representation with VGG}

The comparison is now carried on for VGG features, including 
the proposed GOTEX-mix model (Fig.~\ref{fig:CNNgen_big_comp2}.b), the GOTEX-VGG model (Fig.~\ref{fig:CNNgen_big_comp2}.c),
its approximation using Sliced Wasserstein loss (SW-VGG in Fig.~\ref{fig:CNNgen_big_comp2}.d) that was also recently studied in~\cite{heitz2021sliced},
and the original Texture Networks from~\cite{ulyanov2016texture} (TexNet in Fig.~\ref{fig:CNNgen_big_comp2}.e).

To begin with the latter, one can observe that the TexNet architecture creates pseudo-periodic patterns that are not visible in the original texture, as for instance in the second and fourth examples of Fig.~\ref{fig:CNNgen_big_comp2}.e.
This limitation is already reported in~\cite{ulyanov2016texture} and has been linked to overfitting with VGG features.
This effect seems to be mitigated when using optimal transport optimization.

Now, as for the image-based optimization results detailed in the previous section, it is interesting to observe that patch-based synthesis can be close to synthesis obtained when using more sophisticated features relying on deep neural networks (comparing for instance Fig.~\ref{fig:CNNgen_big_comp}.b to Fig.~\ref{fig:CNNgen_big_comp2}.c).
This is also true when considering Sliced Wasserstein approximation (Fig.~\ref{fig:CNNgen_big_comp}.c vs Fig.~\ref{fig:CNNgen_big_comp2}.d).
More precisely, multi-scale patch distributions are more effective to capture long range patterns, while VGG features manage to synthesize photo-realistic high-resolution details.

We used the histogram equalization technique discussed previously for GOTEX-VGG to avoid color inconsistencies.
The benefit of mixing features is here less striking than for single image optimization shown in Section~\ref{ssec:comp_loss}, as the generated samples are very close to the ones obtained by using only VGG features. 
Besides, the stochastic SW approximation results in noticeable decrease in synthesis quality, regarding artifacts (blur, color inconsistencies and high frequency artifacts) and long range correlations.

\paragraph{Adversarial techniques}
For the sake of completeness, we also present results from adversarial techniques that  simultaneously train a discriminative network on generated patches. 
The first method is SinGAN~\cite{shaham2019singan} (Fig.~\ref{fig:CNNgen_big_comp2}.f), a recent generative adversarial network (GAN) technique generating images from a single example and relying on patch sampling.
Note that this method, mainly focused on image reshuffling, learns implicitly during training to copy the border of the example texture by combining an adversarial loss with a least square criterion.
This effect can still be noticed when generating a larger sample, especially in the corners.
The second method is PSGAN~\cite{bergmann2017learning} (Fig.~\ref{fig:CNNgen_big_comp2}.g), a previous approach that similarly adapts the GAN framework to the training of a single image, purely based on an adversarial loss.
As shown here, GAN struggles to train on small samples (here $256 \times 256$ pixels), which results in a lot of noticeable high frequency artifacts. Such high-frequency artifacts with GANs were already reported in~\cite{raad2018survey}.

%
\section{Conclusion and discussion}

In this work, we proposed a general framework for texture synthesis by optimization that allows to constrain the distributions of high-dimensional features through the use of optimal transport distances.
The main contribution of this work is to exploit the semi-dual formulation of optimal transport in order to get a min-max problem with an inner concave maximization problem.
This min-max problem can be solved with an efficient alternate algorithm.
Besides, we provided explicit formulae for the gradients of this functional, which are useful to study and interpret the optimization algorithm.
Contrary to previous methods based on approximations of the optimal transport (like the Sliced Wasserstein distance), the stochastic algorithm used here for the inner problem is guaranteed to converge towards the true optimal transport cost, which makes the global learning process more accurate.
Another interest of this framework is that it is adapted to the learning of a generative model.
As we have seen, such a formulation encompasses the case where one wishes to generate a single image (by optimizing directly on the pixel values) and the case where one wishes to learn a convolutional neural network that can later serve for on-the-fly synthesis.
Experiments showed that both cases lead to high-quality synthesized textures.
Since this method can naturally deal with various sets of features, we were able to compare synthesis results obtained by using patches or features extracted from a pre-learned neural network.
This comparison showed that using multiscale patch distributions is sufficient to synthesize a wide class of textures, while avoiding low-level artifacts (e.g. drifts in the color distribution).
Finally, the image-based optimization can be easily adapted to other applications, and we provided successful results of texture inpainting and texture interpolation.

This work also raises several questions, both on the theoretical and practical aspects.
While the convergence of the inner maximization algorithm is proved, the convergence of the alternate algorithm (which we observed empirically) remains to be investigated.
In particular, it would be interesting to see if one can justify the use of the algorithm alternating single gradient steps on each variable.
On the practical side, one may try to adapt the proposed framework to more general image synthesis problems where the database is much larger.
Indeed, one of the limitations of this work is that the number of variables in the semi-dual optimal transport problem corresponds to the number of points in the target distribution.
Therefore, in order to deal with a much larger database, one may either resort on a batch strategy to estimate the gradient on $\theta$, or to work with another parameterization of the dual variables (as in~\cite{seguy2018large}).
Finally, the different feature-based losses proposed in this work could be used for quantitative evaluation of texture synthesis methods.
The design of a quality measure for texture synthesis requires a more thorough comparison of the various possible features, and may also be driven by a precise perceptual study in order to wisely combine the chosen features, and to validate the resulting criterion.

\end{document}